\documentclass{article}

\usepackage{microtype}
\usepackage{graphicx}
\usepackage{subcaption}
\usepackage[utf8]{inputenc}
\usepackage[T1]{fontenc}
\usepackage{hyperref}
\usepackage{url}
\usepackage{booktabs}
\usepackage{amsfonts}
\usepackage{amsmath, amssymb, amsthm}
\usepackage{nicefrac}
\usepackage{microtype}
\usepackage{xcolor}
\usepackage{graphicx}
\usepackage{algorithm}
\usepackage{algpseudocode}
\usepackage{multirow}

\usepackage[preprint]{arxiv_2026}

\usepackage{amsmath}
\usepackage{amssymb}
\usepackage{mathtools}
\usepackage{amsthm}
\usepackage{bm}
\usepackage{enumitem}

\usepackage{algorithm}

\usepackage[capitalize,noabbrev]{cleveref}

\allowdisplaybreaks

\theoremstyle{plain}
\newtheorem{theorem}{Theorem}
\newtheorem{proposition}[theorem]{Proposition}
\newtheorem{lemma}[theorem]{Lemma}
\newtheorem{corollary}[theorem]{Corollary}
\newtheorem{definition}[theorem]{Definition}
\newtheorem{remark}[theorem]{Remark}
\newtheorem{assumption}[theorem]{Assumption}

\newcommand{\thmsubhead}[1]{\smallskip\noindent\textit{#1.}\quad}
\newcommand{\Mscr}{\mathcal{M}}
\newcommand{\Mlow}{\underline{C}}
\newcommand{\Mhigh}{\overline{C}}
\newcommand{\hmax}{h_{\max}}

\title{Bracketing Uncertainty in Clustering Under the Manifold Hypothesis}

\author{
  Savik Kinger \\
  Department of Computer Science\\
  Yale University
  \And
  Luciano Dyballa \\
  School of Science \& Technology \\
  IE University 
  \And
  Steven W. Zucker\\
  Depts. of Computer Science and Biomedical Engineering\\
  Wu Tsai Institute\\
  Yale University
}

\begin{document}

\maketitle

\begin{abstract}
The manifold hypothesis suggests a natural criterion for clustering: partition data according to the manifold component from which each point is drawn. Whether two components are separable depends on a geometric tradeoff: the ambient separation between components versus the largest gap in sampling. In practice, this tradeoff is rarely assessed explicitly, leading standard methods to over-commit to a single clustering assignment even when the data do not support a unique answer. We formalize this tradeoff by combining intrinsic manifold geometry (volume growth and reach) with sample-level quantities (fill distance and density), yielding a threshold phenomenon for mutual-$k$-nearest-neighbor graphs: when the offset-to-fill ratio exceeds a conservative upper threshold, component separation is preserved; below a lower threshold, components fuse. The gap between these thresholds defines a \emph{geometric uncertainty zone} in which the number of clusters is not identifiable from the data. Nevertheless, conventional approaches still seek one: sweeping parameters (an engineering approach) or fitting a generative mixture model (a model-based approach). Rather than forcing a single estimate of the number of clusters, we propose Manifold-Based Clustering (MBC), which returns an explicit bracket interval to quantify the underlying data uncertainty. This bracket acts as an empirically calibrated diagnostic: it narrows when a single resolution is supported, widens when multiple resolutions coexist, and collapses to one when no separated structure is detectable. Empirically, we find that many real datasets lie within the uncertainty zone rather than admitting one clear answer. Our results suggest that ambiguity in cluster number is often intrinsic, and should be quantified rather than resolved.
\end{abstract}

\section{Introduction}
\label{sec:intro}

\begin{figure}[t]
\centering
\includegraphics[width=\linewidth]{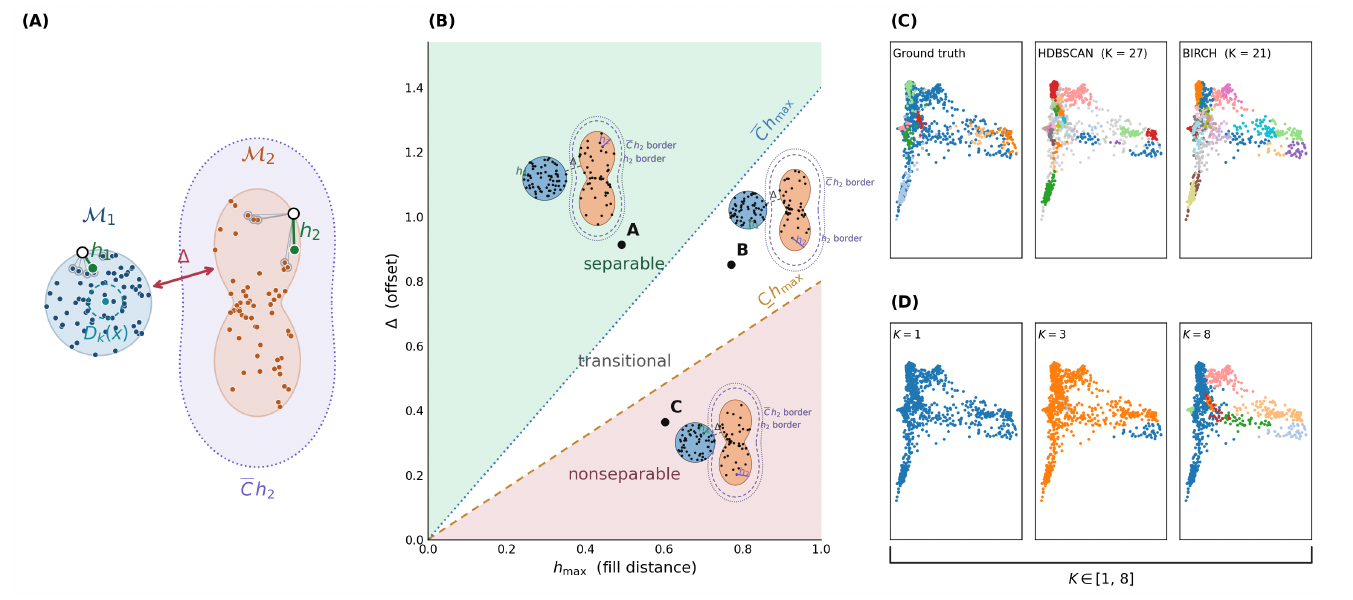}
\caption{\textbf{(A)} Two manifold components $\Mscr_1, \Mscr_2$ with i.i.d.\ samples. The fill distance $h_i$ is the largest sampling gap on a component (green); the offset $\Delta$ is the smallest cross-component distance (red); each sample's $k$NN ball $D_k(x)$ (teal) lies inside the $\Mhigh\,h_2$ boundary (purple), so when $\Delta$ exceeds the $\Mhigh\,h_2$ radius no $k$NN edge can bridge components---the no-bridge bound of Theorem~\ref{thm:knn-threshold}.
\textbf{(B)} The threshold theorem partitions the $(h_{\max}, \Delta)$ plane into three regimes: separable above $\Mhigh\,h_{\max}$, nonseparable below $\Mlow\,h_{\max}$, and a transitional band between them where the cluster count is not identifiable.
\textbf{(C)} Retinal-ganglion-cell recordings: against the ground truth estimate ($K \leq 8$), HDBSCAN ($K = 27$) and BIRCH ($K = 21$) both overshoot.
\textbf{(D)} On the same data MBC returns the bracket $K \in [1, 8]$ rather than a single answer; partitions at $K = 1, 3, 8$.}
\label{fig:manifold_sampling}
\vspace{-4mm}
\end{figure}

Clustering is a notoriously thorny problem. Results depend on the criterion used \citep{kleinberg2002impossibility}, on what counts as separation \citep{hennig2015true}, and on the sample at hand \citep{tibshirani2001estimating}. Practitioners typically fall back on either domain knowledge (e.g., genomics \citep{eisen1998cluster}) or off-the-shelf algorithms \citep{mcinnes2018umap, ester1996dbscan, ankerst1999optics, campello2013hdbscan, campello2015hdbscan}. The statistical power of those algorithms is difficult to assess \citep{dalmaijer2022statistical}, and applying any of them without inspection is risky \citep{chari2023specious}. This is especially true in neuroscience \citep{button2013power}, where even deciding whether the data are clustered at all remains open \citep{dyballa2024functional}. 

The difficulty is not only that different algorithms return different partitions. It is that different approaches make different commitments about what uncertainty in clustering means. One common response is to sweep parameters---for example a neighborhood radius, a minimum cluster size, or a target number of clusters---and inspect how much the resulting partition changes. This treats uncertainty as sensitivity to an algorithmic choice. A model-based approach instead specifies a likelihood family, such as a Gaussian mixture \citep{sanjeevLearningMixturesArbitrary2001}, and turns the problem into model selection. A Bayesian approach places a prior over partitions or mixture components and summarizes the posterior through mass over potential cluster counts \citep{wade2018bayesian}. An information-theoretic approach asks how many clusters can be resolved before finite-sample corrections dominate the information gained by adding more clusters \citep{still2004howmanyclusters}. These approaches are principled within their own frameworks, but they encode different notions of uncertainty: stability, posterior uncertainty, predictive fit, or model complexity.

In this paper, we adopt a geometric perspective to address this problem of learning the ``clusterability'' of data. The manifold hypothesis posits that most data live on a low dimensional manifold whose structure is meaningful and informative for analysis \citep{meilaManifoldLearningWhat2023}. Many popular unsupervised learning algorithms rely on this structure in practice (e.g. \citep{mcinnes2018umap, belkin2003laplacian, maaten2008tsne}). Under the manifold hypothesis, high-dimensional observations are modeled as samples from a compact subset $\Mscr\subset\mathbb{R}^D$ that is either a single connected $C^2$ submanifold or a finite union of disjoint $C^2$ components \citep{feffermanFittingManifoldData2023}. This allows us to define a cluster as precisely a connected component, and thus we can recast clustering as the following \emph{decision problem}: given i.i.d.\ samples $X = \{x_i\}_{i=1}^n$ from a distribution supported on $\Mscr$, can we decide whether the support is connected or decomposes into separated components? 
In this paper, we show this decision depends on a single threshold, and we define a bracket to capture the regime in which that threshold is uncertain.

We estimate this threshold from $k$-nearest-neighbor ($k$NN) graphs and leverage the ratio $\rho$ between two quantities: the \emph{offset} $\Delta$—the smallest Euclidean distance between any two components—and the \emph{fill distance} $h$ of the sample—the size of the largest gap in sample coverage. A smaller fill distance means denser, more uniform sampling. So a large $\rho$ means clusters are well-separated relative to how finely the data are sampled; a small $\rho$ means sampling is too coarse to distinguish nearby components, and they will be linked into a single blob. Classic random geometric graph results justify reading clustering through this lens: such graphs exhibit a sharp connectivity transition as the neighborhood scale grows with the sample size $n$, becoming connected around radii $r_n \asymp ((\log n)/n)^{1/d}$ or when the $k$NN parameter scales as $k \asymp \log n$, under mild regularity conditions \citep{penrose2003random,balister2005connectivity}. We translate this transition from connectivity within a single component to separation between distinct manifold components. The constants in the translation depend only on standard intrinsic geometry—two-sided volume growth (bounds on small-ball volumes) and positive reach \citep{niyogiFindingHomologySubmanifolds2008a}—and the resulting statement quantifies when distinct components remain disconnected in a $k$NN graph rather than linked by spurious ``bridging'' edges.

We further leverage this theory into a practical diagnostic algorithm (MBC). Our construction relies on the mutual $k$NN graph, in which an edge $\{i,j\}$ exists only when $i$ and $j$ each appear in the other's top-$k$ list. 
Mutual-$k$NN graphs let us consider multiple scales of $k$, which we can then locally regularize by the local densities intrinsic to the data, and then define a specific filtration of graphs where the range of $k$ chosen is informed by a heuristic based on the theory. This graph filtration is analogous to those used in the persistent homology literature \citep{carlsson2010characterization}. We study the connected components of this filtration, analogous to the 0th homology group, as a proxy for clusters in the data. Our separation criteria provides a principled approach for defining the range of scales $k$, and the mutual-$k$NN construction ensures that merge events seen during the sweep more likely correspond to genuine geometric proximity rather than artifacts of spurious samples. We make the description of our method precise in \S\ref{sec:algorithm} and empirically justify its performance in \S\ref{sec:experiments}.

Our contributions are threefold. First, we prove a conservative geometric threshold theorem for mutual-$k$NN graphs on separated manifold components (Theorem~\ref{thm:knn-threshold}): above an offset-to-fill threshold, cross-component edges are excluded; below a lower threshold, bridge edges appear. Second, we turn the gap between these bounds into a finite-sample bracket estimator (\S\ref{sec:algorithm}). Third, we evaluate the bracket as an uncertainty diagnostic across synthetic, real-world, and neural datasets (\S\ref{sec:experiments}), with ablations to evaluate the role of the chosen scales and pruning procedures.

\section{Background}
\label{sec:background}

\textit{Neighborhood graphs at the connectivity threshold.} Manifold-learning and spectral methods rely on $k$NN graphs, which depend on choosing $k$ near the intrinsic sampling scale \citep{tenenbaum2000isomap, roweis2000lle, belkin2003laplacian, coifmanDiffusionMaps2006, luxburg2007tutorial, mcinnes2018umap}. Random graph theory specifies this scale: such graphs exhibit a sharp connectivity transition at $k \asymp \log n$, with constants determined by dimension and local volume regularity \citep{penrose2003random, balister2005connectivity}. Under two-sided volume growth, the fill distance $h(R, \Mscr) = \sup_{x\in\Mscr}\min_{r_i \in R}\|x - r_i\|$ and the $k$NN radii $D_k(x) \asymp (k/n)^{1/d}$ are tightly coupled \citep{niyogiFindingHomologySubmanifolds2008a, boissonnat2018book}. Our method uses the mutual-$k$NN graph, which is a sparser, more selective alternative that inherits the same logarithmic threshold up to a constant \citep{brito1997connectivity}. Setting $k = O(\log n)$ thus puts the graph near this threshold, where the offset-to-fill criterion can be applied to decide on separated components by comparing the inter-component distance $\Delta$ to $\hmax$.

\textit{Reach, curvature, and density-based clustering.}
The reach $\tau_M$ of $M\subset\mathbb R^D$ is the infimum distance from $M$ to its medial axis \citep{federer1959curvature}; it controls the scale below which the manifold remains locally regular, accounting for curvature and narrow within-component bottlenecks \citep{aamari2019estimating,boissonnat2023reach}. In our theorem, this dependence appears through the local-mass scale $r_\ast$ and the constants $\underline c,\overline c$. The ratio $\Delta/h$ instead compares the gap between distinct components with the resolution of the observed sample, rather than relaxing or estimating reach.
Density-and threshold-based methods (DBSCAN, OPTICS, BIRCH, HDBSCAN) cluster by
thresholding density or mutual-reachability graphs but rely on user-specified
parameters that implicitly decide whether bridges persist
\citep{ester1996dbscan, ankerst1999optics, campello2013hdbscan,
campello2015hdbscan, Miron1996BIRCH}. 
Adaptive-neighborhood methods estimate a local $k_i$ for each point from sampling density rather than fixing a single global scale \citep{dyballa2023ian}.
We instead follow connected components across a nested family of mutual-$k$NN graphs, giving the construction a natural interpretation in terms of $0$ homology group of persistence diagram \citep{carlsson2010characterization}.

\section{Geometric Cluster-Separation Criterion}
\label{sec:geom-criterion}

The framework borrows two intuitions from Gaussian mixture models: the \emph{offset} between manifold components plays the role of inter-cluster distance, and the \emph{fill distance} of the sample plays the role of within-cluster dispersion. Suppose the data lie on the union
\begin{equation*}
\Mscr \;=\; \Mscr_1 \cup \cdots \cup \Mscr_K, \qquad \Mscr_i \cap \Mscr_j = \emptyset, \;\forall i \neq j \in \{1,\ldots,K\}
\end{equation*}
where each $\Mscr_k$ is a connected manifold component in $\mathbb{R}^D$. Let
$
\Delta \;=\;
\min_{k\neq \ell}\,\bigl\{\|x-y\| \;:\; x\in\Mscr_k,\,y\in\Mscr_\ell\bigr\}
$
denote the \emph{offset} (minimal ambient distance) between distinct components. We define the fill distance for the sampled approximation of $\Mscr$ as follows.

\begin{definition}[Fill distance]
\label{def:filldist}
Let $R = \{r_i\}_{i=1}^n \subset \Mscr$ be a finite point set. The \emph{fill distance} is
$
h_{R,\Mscr} \;=\; \sup_{x\in\Mscr}\min_{1\le i\le n}\|x - r_i\|.
$
A smaller fill distance indicates that $R$ provides a denser covering of $\Mscr$.
\end{definition}

Analogous to the sampling-density and within-component variance criteria in Gaussian mixtures, $h_{R,\Mscr}$ measures sampling dispersion: smaller values correspond to denser, more uniform sampling, which is what manifold-learning algorithms need in order to reliably approximate geodesic distances. 

For brevity write $h$ for $h_{R,\Mscr}$, and consider the ratio $\rho = \Delta/h$.

\subsection{Manifold separation criterion}

The next theorem establishes a threshold phenomenon for the connectivity of mutual-$k$NN graphs on points sampled from disjoint, compact, $d$-dimensional Riemannian manifolds. As $k$ varies, the graph undergoes a transition: when manifolds are far apart relative to sampling density, no edges cross between them; when they are close enough, bridging edges appear with high probability. Under the assumption that clusters correspond to distinct manifold components, the theorem characterizes how the fill distance $h$ (sampling density) and the offset $\Delta$ (intrinsic separation) together determine whether components remain disconnected or fuse in the graph. The dimension $d$ that appears throughout is the \emph{intrinsic} dimension of $\Mscr$, which can be much smaller than the ambient dimension $D$; the threshold constants depend on $d$, not on $D$, so the bracket is most informative when ambient distance reflects intrinsic geometry (a point we revisit in \S\ref{sec:limits}).
\begin{theorem}[Threshold for manifold separation in the mutual-$k$NN graph]
\label{thm:knn-threshold}
Let $\Mscr_1,\Mscr_2 \subset \mathbb{R}^D$ be disjoint, compact, connected, $d$-dimensional $C^2$ submanifolds with positive reach. For each $i \in \{1,2\}$, let $\mu_i$ be a Borel probability measure supported on $\Mscr_i$. Assume there exist constants $0 < \underline c \le \overline c < \infty$ and $r_\ast > 0$ such that $\underline c\,r^d \le \mu_i(B(x,r)) \le \overline c\,r^d$ for every $x \in \Mscr_i$ and $0 < r \le r_\ast$.

\thmsubhead{Sampling and graph}
Independently draw $S_i = \{X_{i1},\ldots,X_{in_i}\}$ with $X_{i1},\ldots,X_{in_i} \stackrel{\mathrm{i.i.d.}}{\sim} \mu_i$, and write $n = n_1+n_2$, $n_{\min} = \min\{n_1,n_2\}$, and $n_{\max} = \max\{n_1,n_2\}$. Define the component-wise fill distances and their maximum by
$$
h_i = \sup_{x \in \Mscr_i}\min_{z \in S_i}\|x-z\|,
\qquad
\hmax = \max\{h_1,h_2\},
$$
and define the ambient offset by
$$
\Delta = \min\bigl\{\|x-y\| : x \in \Mscr_1,\ y \in \Mscr_2\bigr\}.
$$
Fix $\delta,\varepsilon \in (0,1)$, set $k = \lceil A\log(4n/\delta)\rceil$ with $A \ge 3/\varepsilon^2$, and form the mutual-$k$NN graph $G_k^{\mathrm{mut}}$. Assume the finite-sample scale conditions stated in App.~\ref{app:mutual-knn-setup}, which keep the  fill distance and $k$NN radii within the local-mass scale $r_\ast$.

\thmsubhead{Threshold statement}
There exist explicit constants $\Mhigh,\Mlow > 0$, given in Rem.~\ref{rem:constants-threshold}, such that:
\begin{enumerate}
\item[(i)] The following
$$
\frac{\Delta}{\hmax} > \Mhigh
\qquad\Longrightarrow\qquad
G_k^{\mathrm{mut}}\text{ has no cross-component edge}
$$
fails with probability at most $3\delta/8$.

\item[(ii)] Fix $a \in (0,1/8)$ and write $B = 1+2a$. If $n_{\max} \le 2n_{\min}$ and $B\Delta \le r_\ast$, then
$$
\frac{\Delta}{\hmax} < \Mlow
\qquad\Longrightarrow\qquad
G_k^{\mathrm{mut}}\text{ contains a cross-component edge}
$$
fails with probability at most
$\delta/4 + 2\exp(-\underline c\,a^d n_{\min}\Delta^d) + \exp(-\gamma k)$,
where $\gamma = \log(4/e)-1/A$.
\end{enumerate}
\end{theorem}

\paragraph{Proof sketch.}
The proof has two halves, one for each direction of the threshold. We first show that the realized fill distance and the $k$NN radii lie on the same sampling scale. The upper direction then follows by comparing every neighborhood radius with the component offset. For the lower direction, we instead find one short observed cross-component pair and show that its endpoints are in each component. Formal statements and proofs are deferred to App.~\ref{app:knn-threshold}.

Write $D_k(z)$ for the distance from a sample $z \in S$ to its $k$th nearest neighbor.
For the lower-threshold direction, let $x_0 \in \Mscr_1$ and $y_0 \in \Mscr_2$ be distance $\Delta$ apart and define the radius-$a\Delta$ neighborhoods
$U = \Mscr_1 \cap B(x_0,a\Delta)$ and $V = \Mscr_2 \cap B(y_0,a\Delta)$, which we call \emph{caps}, and write $B = 1+2a$. We say the caps are occupied when $S_1\cap U$ and $S_2\cap V$ are each nonempty. This guarantees a cross-component pair at distance at most $B\Delta$. For $z \in S_i$, ``same-component crowding'' refers to the number of points in $S_i\setminus\{z\}$ within distance $B\Delta$ of $z$. The $\delta$-terms bound the fill-distance and $k$NN-radius estimates. 

\emph{Step 1: identify the common sampling scale.} We first use a covering argument to rule out holes much larger than the sampling scale and a packing argument to show that a hole of that scale remains. Together, these bounds show that $\hmax$ is of order $\bigl(\log(n_{\min}/\delta)/n_{\min}\bigr)^{1/d}$ with high probability. A separate Chernoff argument shows that every sample point has $k$ neighbors within a radius of order $(k/n_{\min})^{1/d}$. Since $k \asymp \log(n/\delta)$, the fill distance and the $k$NN radii lie on the same scale.

\emph{Step 2: no cross-component edges, proving (i).} If $\Delta/\hmax > \Mhigh$, then $\Delta > \Mhigh\,\hmax \ge D_k(z)$ for every $z \in S$. In particular, for any $z \in S_1$ and $w \in S_2$,
$\|z-w\| \ge \Delta > D_k(z)$, so $w \notin N_k(z)$. The same uniform radius bound gives $\|z-w\| \ge \Delta > D_k(w)$, so $z \notin N_k(w)$. Thus no directed neighbor relation crosses between the components. The union-$k$NN graph therefore has no cross-component edge, and its subgraph $G_k^{\mathrm{mut}}$ inherits the same conclusion.

\emph{Step 3: a mutual cross-edge under $\Delta/\hmax < \Mlow$, proving (ii).} Consider the closest-point caps $U$ and $V$ defined above. We combine three properties:
(\emph{a})~\textit{Cap occupancy.} Since $a\Delta < B\Delta \le r_\ast$, the local lower mass bound gives $\mu_1(U),\mu_2(V) \ge \underline c\,(a\Delta)^d$. Each cap is therefore missed by its corresponding sample with probability at most $\exp(-\underline c\,a^d n_{\min}\Delta^d)$. A union bound shows that both caps are occupied except on probability $2\exp(-\underline c\,a^d n_{\min}\Delta^d)$. Whenever this occurs, we may choose $x \in S_1\cap U$ and $y \in S_2\cap V$, and the triangle inequality gives $\|x-y\| \le a\Delta+\Delta+a\Delta = B\Delta$. Thus the closest observed cross-component pair is no farther apart than $B\Delta$.
(\emph{b})~\textit{Control of same-component competitors.} The high-probability upper bound on $\hmax$, together with $\Delta/\hmax < \Mlow$ and the balance-sampling assumption $n_{\max} \le 2n_{\min}$, implies that for every $z \in S_i$, the conditional expected number of points in $S_i\setminus\{z\}$ within distance $B\Delta$ is at most $k/4$. A Chernoff bound for this binomial count shows that reaching $k$ is exponentially unlikely. After a union bound over all sample points, every endpoint has at most $k-1$ same-component competitors within $B\Delta$, except on probability $\exp(-\gamma k)$.
(\emph{c})~\textit{Reciprocity from a closest observed pair.} Among the closest observed cross-component pairs, choose $(x^\star,y^\star)$ according to an ordering placed on pairs of points by distance. Viewed from $x^\star$, no point of $S_2$ can come before $y^\star$ as a closer point would contradict the ordering. Viewed from $y^\star$, the same argument prevents any point of $S_1$ from preceding $x^\star$. Since there are at most $k-1$ such predecessors, $y^\star \in N_k(x^\star)$ and $x^\star \in N_k(y^\star)$, so $\{x^\star,y^\star\}$ is a mutual cross-component edge.

The proof combines standard ideas from manifold sampling and neighborhood-graph clustering. The covering and packing control of sample coverage follows the reach-based manifold-sampling framework of \citet{niyogiFindingHomologySubmanifolds2008a}, while the logarithmic neighborhood scale and local occupancy bounds are classical in random geometric graph theory \citep{penrose2003random}. The use of mutual-$k$NN connectivity for clustering builds on \citet{brito1997connectivity} and the cluster-identification analysis of \citet{maier2009mutualknn}. Full details are in Appendix~\ref{app:knn-threshold}.
\begin{remark}[Scaling of the thresholds]
\label{rem:constants-threshold}
Let $R_n = \log(4n/\delta)/\log(n_{\min}/\delta)$ and $M = (\overline c/\underline c)^{1/d}$. The explicit constants from the proof are
$$
\Mhigh = \left(\frac{12 A R_n\,\overline c}{(1-\varepsilon)\underline c}\right)^{1/d},
\qquad
\Mlow = \left(\frac{A R_n\,\underline c}{2^{d+4}\,\overline c\,B^d}\right)^{1/d}.
$$
In particular, $\Mhigh = \Theta((A R_n)^{1/d}M)$ and $\Mlow = \Theta((A R_n)^{1/d}/(BM))$. Under balanced sampling ($R_n \approx 1$), fixed $\varepsilon$ and $a$, and bounded geometry ($\overline c/\underline c = \Theta(1)$), both thresholds are $\Theta(A^{1/d})$ with constants depending only on $d$.
\end{remark}

We further show an extentsion Theorem~\ref{thm:knn-threshold} to the noisy data case. Specifically, we show in App.~\ref{app:knn-threshold}, \emph{i})~Lemma~\ref{lem:mutual-knn-sample-offset} the observed cross-component distance for a given sampling $\Delta_{\mathrm{sam}}$ can be bounded as:
$$
\Delta \le \Delta_{\mathrm{sam}} \le \Delta+h_1+h_2 \le \Delta+2\hmax.
$$
This motivates estimating the offset-to-fill ratio from the data. 

\section{From thresholds to a bracket estimator}
\label{sec:algorithm}

Theorem~\ref{thm:knn-threshold} characterizes, for a fixed $k$NN scale and known geometric constants, when the resulting graph can and cannot bridge distinct manifold components. More precisely, Theorem~\ref{thm:knn-threshold} identifies this scale to be
$\lceil A\log(4n/\delta)\rceil,$ where $A$ is the scale-determining constant. Yet in practice, this scale is unknown. Consider however that the threshold statement can be read in reverse. Because $\Mhigh$ and $\Mlow$ depend on $A$ and $d$, an empirical estimate of $\rho = \Delta/\hmax$ lets us invert the threshold to determine the range of plausible $A$, and hence of $k$, over which neither threshold direction resolves whether the graph preserves or bridges the component separation. We leverage this observation to build a family of graphs, one per scale, whose component counts define a range of plausible clustering assignments.

MBC does this by first using a carefully constructed pilot graph to supply a baseline estimate of the observed $\rho$, denoted $\widehat\rho$. We then invert the two practical threshold curves at $\widehat\rho$ to obtain the range of graph scales associated with this transition. Finally, we rebuild the mutual-$k$NN graphs across that range and record their connected-component count at each scale. The resulting range of counts forms the MBC bracket for the number of clusters supported by the sample.

We now describe the four steps of this construction.


\paragraph{Step 1: Set the pilot scale.}
For a fixed $\delta$, set
$ A_0:=1 $
and
$ k^\star:=\lceil A_0\log(4n/\delta)\rceil
=\lceil\log(4n/\delta)\rceil. $
Compute the $k^\star$ nearest neighbors of each observation and record
$H_i^{\mathrm{pilot}}:=D_{k^\star}(x_i).$
Theorem~\ref{thm:knn-threshold} identifies the scale family
$k(A)=\lceil A\log(4n/\delta)\rceil,$
for which the fill distance and the $k$NN radii lie on the same sampling scale. The theorem uses a larger coefficient to obtain a conservative finite-sample guarantee; we use $A_0=1$ and evaluate this choice in App.~\ref{app:A-sweep}.

\paragraph{Step 2: Construct the pilot graph.}
We first remove observations whose pilot radii are overly large. Define
\begin{equation}
\label{eq:density-prune}
\tau_q
:=
\alpha_q\,
\operatorname{Quantile}_q
\bigl\{
H_j^{\mathrm{pilot}}:j\in[n]
\bigr\},
\qquad
\mathcal A
:=
\bigl\{
i\in[n]:
H_i^{\mathrm{pilot}}\le\tau_q
\bigr\}.
\end{equation}
A large pilot radius means that an observation must reach farther than most of the sample to collect $k^\star$ neighbors, making its incident edges plausible sources of spurious bridges. Mutuality reduces this risk but does not remove it.

We then adjust the pilot degree for variation in local sampling density:
\begin{equation}
\label{eq:local-k}
k_i^\star
=
\max\!\left\{
k^\star,\,
\left\lfloor
k^\star
\left(
\frac{H_{\mathrm{ref}}}
{H_i^{\mathrm{pilot}}}
\right)^{d_{\mathrm{eff}}}
\right\rfloor
\right\},
\qquad
H_{\mathrm{ref}}
=
\operatorname{median}
\bigl\{
H_j^{\mathrm{pilot}}:
H_j^{\mathrm{pilot}}>0
\bigr\},
\end{equation}
with clipping details given in App.~\ref{app:impl}. Observations in denser regions receive larger neighbor lists, reducing variation in the graph radii caused only by local density.

We recompute the neighbor lists on $\mathcal A$ at degrees $k_i^\star$ and form the mutual-neighbor graph. An additional nearest-neighbor fallback is added only for observations left isolated by the reciprocal construction. Write
$\widehat{\mathcal C}_1,\ldots,\widehat{\mathcal C}_{K_0}$
for its connected components and $H_i$ for the resulting neighborhood radius of retained observation $i$. These components are specific to the pilot partition used in Step~3.

\paragraph{Step 3: Estimate the ratio and set the sweep range.}
If the pilot graph has more than one component, define
$$
\widehat\Delta
:=
\min_{r\neq s}
\min_{\substack{
i\in\widehat{\mathcal C}_r\\
j\in\widehat{\mathcal C}_s
}}
\|x_i-x_j\|,
\qquad
\widehat h
:=
\operatorname{Quantile}_{0.5}
\bigl\{
H_i:i\in\mathcal A,\ H_i>0
\bigr\},
\qquad
\widehat\rho
:=
\frac{\widehat\Delta}{\widehat h}.
$$
If the pilot graph is connected, $\widehat\Delta$ and $\widehat\rho$ are left undefined. Because the same pilot graph supplies both the partition and the radii, $\widehat\rho$ asks whether the gaps visible at the pilot scale are large relative to the neighborhoods required to form that graph. For the true component partition, Lemma~\ref{lem:mutual-knn-sample-offset} gives
$
\Delta\le\Delta_{\mathrm{sam}}\le\Delta+2\hmax,
$
which motivates the observed gap in the numerator. We therefore treat $\widehat\rho$ as an empirical regime locator, rather than as a sharp estimator of $\Delta/\hmax$.

Write
$
\underline C(A,d_{\mathrm{eff}})
$
and
$
\overline C(A,d_{\mathrm{eff}})
$
for the practical lower bridge and upper no-bridge curves defined in App.~\ref{app:impl}. They preserve the $A^{1/d_{\mathrm{eff}}}$ scaling of the theoretical thresholds in Rem.~\ref{rem:constants-threshold}. In the transitional regime,
$
\underline C(A_0,d_{\mathrm{eff}})
<
\widehat\rho
<
\overline C(A_0,d_{\mathrm{eff}}),
$
we choose the endpoints by solving
$$
\underline C(A_{\mathrm{high}},d_{\mathrm{eff}})
=
\widehat\rho,
\qquad
\overline C(A_{\mathrm{low}},d_{\mathrm{eff}})
=
\widehat\rho.
$$
Since both curves increase with $A$, the upper curve is reached at a smaller scale than the lower curve. 
The two solutions therefore mark the beginning and end of the range in which the practical threshold map does not resolve the bridge transition. Inverting the curves in this way turns the gap between the two fixed-scale bounds into a range of graph scales adapted to the observed pilot geometry.

Outside this transitional case, the sweep is chosen according to the same geometry. If $\widehat\rho$ lies above the upper curve, the pilot partition is already well separated at $A_0$, and a narrow range around the pilot scale tests whether that partition persists under nearby choices of $k$. If $\widehat\rho$ lies below the lower curve, or if the pilot graph is connected, the sweep instead extends toward smaller degrees, since increasing $A$ only adds edges and cannot reveal a separation that is absent at the pilot scale. The exact endpoint rules and clipping are given in App.~\ref{app:impl}; their behavior is illustrated by the ablations in Apps.~\ref{app:A-sweep} and~\ref{app:perturbation-walk}.

Finally, convert the coefficient range to graph degrees:
\begin{equation}
\label{eq:bracket-endpoints}
k_{\mathrm{low}}
=
\left\lceil
A_{\mathrm{low}}\log\frac{4n}{\delta}
\right\rceil,
\qquad
k_{\mathrm{high}}
=
\left\lceil
A_{\mathrm{high}}\log\frac{4n}{\delta}
\right\rceil.
\end{equation}

\paragraph{Step 4: Sweep the graph scales and compute the bracket.}
We evaluate an integer grid over
$\mathcal K\subseteq[k_{\mathrm{low}},k_{\mathrm{high}}].$
At each $k\in\mathcal K$, rescale the local degrees as before, with details in App.~\ref{app:full-algorithm}, and rebuild the mutual-neighbor graph. We continue to denote this graph by $G_k^{\mathrm{mut}}$. Since every $k_i(k)$ is non-decreasing in $k$, these graphs create a filtration.

Across this range, increasing $k$ adds neighbor relations and carries the filtration from finer toward coarser graph partitions. The threshold has already been used to determine which part of this filtration corresponds to the uncertain bridge transition. At each selected scale, we therefore read the candidate clusters directly from the connected components of $G_k^{\mathrm{mut}}$ and define
$$ K_{\mathrm{raw}}(k)
=
\#\operatorname{Comp}
\bigl(
G_k^{\mathrm{mut}}
\bigr),
\qquad
K_{\mathrm{big}}(k)
=
\#\left\{
C\in\operatorname{Comp}
\bigl(
G_k^{\mathrm{mut}}
\bigr):
|C|\ge s_{\min}
\right\}.
$$
The primary MBC bracket is therefore
\begin{equation}
\label{eq:bracket-def}
\bigl[
K_{\mathrm{low}},
K_{\mathrm{high}}
\bigr]
=
\left[
\min_{k\in\mathcal K}K_{\mathrm{big}}(k),
\;
\max_{k\in\mathcal K}K_{\mathrm{big}}(k)
\right].
\end{equation}

We also report the mass-bounded count
\begin{equation}
\label{eq:K-mass}
K_{\mathrm{mass},\gamma}(k)
=
\min\left\{
m:
\sum_{i=1}^{m}s_{(i)}(k)
\ge
\gamma|\mathcal A|
\right\},
\qquad
\gamma=0.95,
\end{equation}
where
$s_{(1)}(k)\ge s_{(2)}(k)\ge\cdots$
are the component sizes. Its range over $\mathcal K$ records how many components are needed to cover most of the retained sample and is reported alongside, rather than in place of, the primary bracket. Point estimates and representative labels are defined from the same sweep in App.~\ref{app:impl}.

\begin{algorithm}[h]
\caption{Bracket estimator (full pseudocode in App.~\ref{app:full-algorithm})}
\label{alg:mbc}
\begin{algorithmic}[1]
\Require $X\in\mathbb R^{n\times D}$, $\delta\in(0,1)$
\State \textbf{Set pilot scale.}
Set $A_0\gets1$ and
$k^\star\gets\lceil\log(4n/\delta)\rceil$;\
compute $H_i^{\mathrm{pilot}}\gets D_{k^\star}(x_i)$
\State \textbf{Construct pilot graph.}
Retain $\mathcal A$ using Eq.~\eqref{eq:density-prune};\
compute $k_i^\star$ using Eq.~\eqref{eq:local-k};\
build the $\alpha$-gated mutual-neighbor pilot graph (App.~\ref{app:impl}) and record its components and radii $H_i$
\State \textbf{Set sweep range.}
Estimate $\widehat\rho$ when the pilot graph is disconnected;\
invert $\overline C$ and $\underline C$ in the transitional regime and otherwise apply the endpoint rule in App.~\ref{app:impl};\
obtain $k_{\mathrm{low}},k_{\mathrm{high}}$ from Eq.~\eqref{eq:bracket-endpoints}
\State \textbf{Compute bracket.}
Rescale $k_i^\star$ over the selected scales;\
rebuild the mutual graphs without fallback or gate;\
record the raw, size-filtered, and mass-bounded component counts
\State \textbf{Return.}
The primary bracket from Eq.~\eqref{eq:bracket-def}, together with the raw and mass-bounded companion brackets
\end{algorithmic}
\vspace{-1mm}
\end{algorithm}

\paragraph{Outputs, parameters, and complexity.}
The algorithm returns the primary size-filtered bracket, a raw-count bracket, and a mass-bounded companion bracket. A narrow primary bracket indicates that one component count remains stable over the selected threshold range; a wide bracket indicates that several resolutions remain supported; a small lower endpoint indicates that the representation does not support persistent separation at the sampled scale.

Every experiment uses
$\delta=0.05,$
$q=0.95,$
$\alpha_q=1.5,$
per-edge slack
$\alpha=1.5,$
$\gamma=0.95,$
$A_0=1,$
and
$s_{\min} = \max\{\lceil0.005|\mathcal A|\rceil,k^\star,5\}.$
These values are held fixed across datasets. Additional implementation details are given in App.~\ref{app:impl}.

Let $s$ be the number of graph scales in the sweep and $k_{\max}$ the largest neighborhood size. With brute-force neighbor search, the worst-case cost is
$O\!\left(s(n^2D+nk_{\max})\right).$

\section{Experiments}
\label{sec:experiments}

We test whether the MBC bracket behaves as Theorem~\ref{thm:knn-threshold} predicts, and whether its behavior arises from the threshold mechanism the theorem describes. Three suites are used: $38$ synthetic datasets in eight families that each stress one structural property the theory addresses (classic shapes, additive noise, background contamination, anisotropy, ambient dimension, multi-scale hierarchy, near-overlap, cluster imbalance; details in App.~\ref{app:synth-catalog}). We further include 11 standard real-world benchmarks across image, text, and tabular data and three neural datasets. To compare the bracket with parameter-varied baselines, we form a \emph{grid bracket} $[\min_g K_g, \max_g K_g]$ over six DBSCAN settings and eight HDBSCAN settings (grids in App.~\ref{app:impl}). We run KMeans, GMM, Ward, and Spectral at $K = K^\star$ when ground truth is known as additional points of reference. MBC parameters are fixed across every dataset at $\delta = 0.05$, $\alpha = 1.5$, $q = 0.95$, $\alpha_q = 1.5$. Three seeds are used and inter-seed variability is reported in App.~\ref{app:multi-seed}. To distinguish a wide-but-correct bracket from a narrow-but-occasionally-wrong one, we define  \emph{informativeness} as the average of the indicator denoting coverage of $K^\star$ divided by (median width $+\,1$).

\paragraph{Mechanism tests.}
The theory ties three parts of the construction to specific predictions: \emph{reciprocity} (mutual vs.\ union $k$NN), \emph{density pruning}, and the \emph{$k$NN scale coefficient} $A$. Per-mechanism ablations in App.~\ref{app:A-sweep} show the main effects. Varying $A$ moves the bracket from fragmentation at small $A$ toward coarser partitions at large $A$. The union graph matches the mutual graph on clean separable data but is sensitive to outliers in the presence of noisy outliers. The density pruning heuristic has a useful, but limited, effect when ambient, background-noise points lie close to the manifold rather than forming clear off-manifold outliers.

\looseness=-1
\subsection{Bracket calibration}
\label{sec:exp-calibration}

On the $38$ synthetic datasets, MBC reports a median bracket width of $0$ with $68\%$ coverage of $K^\star$. 
HDBSCAN-grid covers $K^\star$ on $92\%$ of datasets at median width $3.5$; DBSCAN-grid covers $84\%$ at median width $6$. Using unrounded coverage and width, the informativeness scores are $0.68$, $0.20$, and $0.12$ for MBC, HDBSCAN-grid, and DBSCAN-grid, respectively.
MBC leads on most synthetic families (Fig.~\ref{fig:per_family_brackets}, App.~\ref{app:synth-catalog}, Tab.~\ref{tab:per-family}). Against the stronger of the two grids, its informativeness is $1.00$ vs.\ $0.14$ on classic shapes, $0.44$ vs.\ $0.25$ on background contamination, and $0.71$ vs.\ $0.21$ on class imbalance. On high-dimensional blobs  the HDBSCAN-grid is ahead, $1.00$ vs.\ $0.88$. MBC often deals with the standard coverage tradeoff: if it commits to a single $K$ on most datasets and it is occasionally wrong, while not committing creates a wide grid which misses less often but offers little informativeness.
A separate $16$-dataset comparison with a variational Dirichlet-process mixture shows the same coverage-width tradeoff. On this analysis, MBC has zero median width. The Bayesian interval has similar mean width but lower coverage, and HDBSCAN attains higher coverage only at substantially wider intervals (App.~\ref{app:additional-results}).

\begin{figure}[t]
\centering
\includegraphics[width=0.98\linewidth]{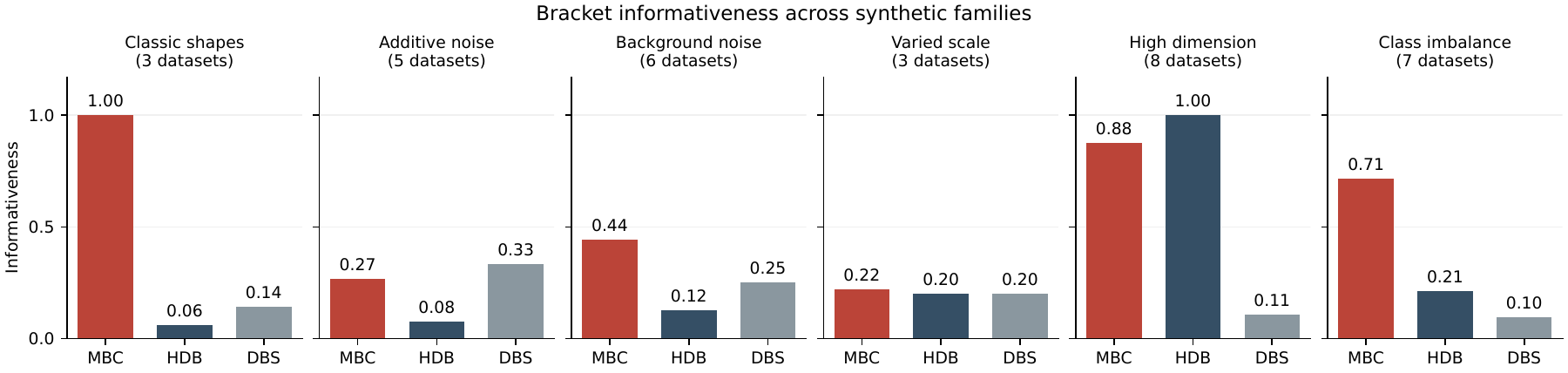}
\caption{Per-family bracket informativeness on six synthetic families of data. Bars are MBC, HDBSCAN-grid, and DBSCAN-grid. MBC is strongest on the families that match the separation criterion and competitive on high-dimensional blobs, where the HDBSCAN grid performs well. Additional families (e.g. adversarial) that violate the single-scale assumption are provided in App.~\ref{app:synth-catalog}.}
\label{fig:per_family_brackets}
\vspace{-4mm}
\end{figure}

\looseness=-1
\subsection{Bracket response to perturbations}
\label{sec:exp-perturbations}

Theorem~\ref{thm:knn-threshold} predicts that the bracket widens as as $\rho = \Delta/\hmax$ drifts into the uncertainty zone and collapses toward $1$ when the sampled representation no longer supports separation. Controlled perturbations and additional geometric stress tests examine this behavior directly (Apps.~\ref{app:perturbation-walk},
\ref{app:additional-results}).

\emph{Contamination} is the hardest axis for the method: reciprocity and density pruning help expose the transition, but background noise is not cleanly separated from the manifold, so the pruning ablation is mixed (App.~\ref{app:A-sweep}).
\emph{Sampling density and ambient dimension} behave as predicted. Shrinking $n$ increases $h_{\max} \asymp ((\log n)/n)^{1/d}$ and drifts the regime toward non-separable, while high-dimensional blob families show that the bracket follows intrinsic geometry rather than $D$ alone (App.~\ref{app:sampling-sweep}).
\emph{Varied curvature and metrics} provide complementary stress tests. On separated
quadratic patches, every bracket contains $K^\star$ across center curvatures
from $-2$ to $2$, with mean widths between $0.3$ and $1.3$; bracket-derived assignments reach ARI
$0.990$ and $0.992$. Alternative metrics preserve coverage on three diagnostic
datasets but can substantially widen the bracket. Tightly wound spirals
and the Swiss roll remain difficult when their local reach is small relative to
the graph scale (App.~\ref{app:additional-results}).

\subsection{Representation sensitivity on real-world data}
\label{sec:exp-real}

Here we test whether the bracket detects representation failure. Theorem~\ref{thm:knn-threshold} predicts that when ambient distance fails to reflect intrinsic geometry, both $\Delta$ and $\hmax$ collapse onto a common $D$-dimensional scale, $\widehat\rho$ concentrates near $1$, and the bracket narrows. The eleven real-world datasets confirm this: $\widehat\rho$ sits in a narrow band near the transitional boundary on all eleven datasets.
\begin{table}[t]
\centering
\small
\caption{Representative bracket behavior across synthetic, real-world, and neuroscience datasets. ``Primary'' is the size-filtered MBC bracket $K_{\mathrm{big}}$; ``mass'' is the $K_{\mathrm{mass},0.95}$ companion bracket. HDBSCAN and DBSCAN brackets are the ranges of non-noise cluster counts over their parameter grids. A baseline bracket containing $0$ means that at least one grid setting labeled every point as noise. $^{\dagger}$ denotes a biological estimate from independent physiological measurements rather than a ground-truth label set.}
\label{tab:main-brackets}
\resizebox{\textwidth}{!}{%
\begin{tabular}{llccccccc}
\toprule
Suite & Dataset & $n$ & $D$ & $K^\star$ & MBC primary & MBC mass & HDBSCAN grid & DBSCAN grid \\
\midrule
\multirow{3}{*}{Synthetic}
& Blobs 2D clean        & 2000 &  2 & 4 & $[4, 4]$ & $[4, 4]$   & $[4, 5]$  & $[4, 6]$ \\
& Moons, noise $0.10$   & 2000 &  2 & 2 & $[1, 4]$ & $[1, 4]$   & $[2, 10]$ & $[0, 2]$ \\
& Blobs 50D easy        & 2000 & 50 & 6 & $[6, 6]$ & $[6, 221]$ & $[6, 6]$  & $[0, 6]$ \\
\midrule
\multirow{3}{*}{Real}
& BreastCancer          &  569 & 30 &  2 & $[1, 2]$ & $[1, 21]$ & $[0, 2]$ & $[1, 1]$ \\
& Iris                  &  150 &  4 &  3 & $[2, 3]$ & $[2, 7]$  & $[2, 4]$ & $[1, 2]$ \\
& Wine                  &  178 & 13 &  3 & $[1, 2]$ & $[1, 12]$ & $[2, 2]$ & $[0, 2]$ \\
\midrule
\multirow{3}{*}{Neuro}
& Retina (full)         & 1146 &  5 & $\approx 8^{\dagger}$ & $[1, 8]$ & $[1, 7]$  & $[3, 27]$ & $[1, 6]$ \\
& Retina (labeled)      &  335 &  5 & 7 & $[1, 7]$ & $[1, 10]$ & $[5, 13]$ & $[1, 5]$ \\
& V1                    &  640 &  6 & 1 & $[1, 1]$ & $[1, 1]$  & $[0, 8]$  & $[1, 1]$ \\
\bottomrule
\end{tabular}}
\vspace{-4mm}
\end{table}

On the tabular datasets the bracket covers $K^\star$ at width $1$ on Iris and BreastCancer. On Wine the primary bracket misses the third grape variety, but the mass companion admits it as a long-tail minority within the wider window $[1, 12]$. The baseline grids fail on BreastCancer in opposite directions: the DBSCAN grid collapses to $[1, 1]$ and misses $K^\star{=}2$, while the HDBSCAN grid covers $K^\star$ but admits an all-noise, zero clusters, assignment, at its lower endpoint. The remaining image and text datasets in the suite sit at the representation-failure boundary, where Theorem~\ref{thm:knn-threshold} predicts the bracket narrows toward 1. This is a finding about the representation capacity of the data rather than a tuning failure of the algorithm.

On ImageNet-100 in the tested representation, MBC under-resolves; but on ImageNet-Dogs it returns $[14,19]$ around $K^\star=15$, against $[2,23]$ for HDBSCAN and $[19,105]$ for DBSCAN, and values drawn from the MBC interval also give the highest mean downstream ARI for both KMeans and spectral clustering (see App.~\ref{app:additional-results} for additional details).

\subsection{Biological results}
\label{sec:neuro}

The retinal and primary-visual-cortex recordings are the kind of data the theory was built for, physiologically characterized but not crisply clusterable, with $K^\star$ either estimated from independent biology or unknown altogether. Each dataset is a diffusion-map embedding of the recordings, and the bracket is computed on the embedding directly so the geodesic structure is preserved.

On the full retinal recording, the bracket $[1, 8]$ contains the physiological estimate of $K^{\star}\approx 8$ retinal-ganglion-cell types based on independent physiological measurements \citep{dyballa_population_2024}. On the labeled retinal subset it returns $[1,7]$, containing $K^{\star}= 7$; HDBSCAN reaches this count only at some grid settings, while DBSCAN stays below it across the grid. On V1 the bracket $[1, 1]$ indicates that the diffusion-map embedding supports no manifold-separable structure.

\section{Discussion}
\label{sec:limits}

Three limitations follow from the hypotheses of Theorem~\ref{thm:knn-threshold}. (i) When contamination is high or clusters are weakly separated, the bracket often collapses to one. Often this is the intended behavior, however---the bracket does not force a partition the data do not support.  (ii) Nearly overlapping clusters may violate the separation hypothesis, so the threshold guarantees do not apply. Other methods still produce a partition here by imposing stronger modeling assumptions, such as a GMM at known $K$; our approach instead makes the ambiguity explicit by identifying when multiple clusterings are equally defensible from the geometry. (iii) Components smaller than $s_{\min}$ are not reported by the primary bracket, which causes MBC to miss a true cluster on Wine (\S\ref{sec:exp-real}). The size threshold requires that a component be supported by enough samples to be distinguished from sampling noise, and the mass-bounded companion bracket retains the small components the primary one filters.

A fourth limitation is more fundamental than the three above: the choice of distance. The threshold constants depend on the intrinsic dimension $d$, so when ambient distance fails to reflect intrinsic geometry, $\Delta$ and $\hmax$ collapse onto a common $D$-dimensional scale and the bracket narrows, signaling that the representation does not expose cluster structure at the relevant scale. Geodesic-aware embeddings can often expose it (diffusion maps, Isomap, UMAP, learned encoders) where the ambient metric does not; the retinal analysis of \S\ref{sec:neuro} uses a diffusion-map embedding and returns a bracket containing the physiological estimate. 

Across synthetic, real-world, and neural data, the bracket trades raw coverage for width and gains informativeness against both parameter-swept baselines. It locates most real datasets inside the uncertainty zone, where the number of clusters is not identifiable from the geometry alone.
Rather than forcing a single clustering, it quantifies when such a choice is justified by the data.

\begin{ack}
This article has received funding from the European Commission's Marie Sk\l{}odowska-Curie Action under grant agreement no. 101207931 (LD).
\end{ack}

\bibliographystyle{plainnat}
\bibliography{references}

@article{coifmanDiffusionMaps2006,
    title = {Diffusion maps},
    journal = {Applied and Computational Harmonic Analysis},
    volume = {21},
    number = {1},
    pages = {5-30},
    year = {2006},
    note = {Special Issue: Diffusion Maps and Wavelets},
    issn = {1063-5203},
    doi = {https://doi.org/10.1016/j.acha.2006.04.006},
    url = {https://www.sciencedirect.com/science/article/pii/S1063520306000546},
    author = {Ronald R. Coifman and Stéphane Lafon}
}

@article{feffermanFittingManifoldData2023,
      title={Fitting a manifold to data in the presence of large noise}, 
      author={Charles Fefferman and Sergei Ivanov and Matti Lassas and Hariharan Narayanan},
      year={2023},
      eprint={2312.10598},
      archivePrefix={arXiv},
      primaryClass={math.ST},
      url={https://arxiv.org/abs/2312.10598}, 
}

@article{dyballa2023ian,
  title = {{{IAN}}: {{Iterated}} Adaptive Neighborhoods for Manifold Learning and Dimensionality Estimation},
  author = {Dyballa, Luciano and Zucker, Steven W.},
  year = 2023,
  month = feb,
  journal = {Neural Computation},
  volume = {35},
  number = {3},
  eprint = {https://direct.mit.edu/neco/article-pdf/35/3/453/2071882/neco\_a\_01566.pdf},
  pages = {453--524},
  issn = {0899-7667},
  doi = {10.1162/neco_a_01566}
}

@misc{meilaManifoldLearningWhat2023,
   author = "Meilă, Marina and Zhang, Hanyu",
   title = "Manifold Learning: What, How, and Why", 
   journal= "Annual Review of Statistics and Its Application",
   year = "2024",
   volume = "11",
   number = "Volume 11, 2024",
   pages = "393-417",
   doi = "https://doi.org/10.1146/annurev-statistics-040522-115238",
   url = "https://www.annualreviews.org/content/journals/10.1146/annurev-statistics-040522-115238",
   publisher = "Annual Reviews",
   issn = "2326-831X",
   type = "Journal Article",
  }

@inproceedings{sanjeevLearningMixturesArbitrary2001,
    author = {Arora, Sanjeev and Kannan, Ravi},
    title = {Learning mixtures of arbitrary gaussians},
    year = {2001},
    isbn = {1581133499},
    publisher = {Association for Computing Machinery},
    address = {New York, NY, USA},
    url = {https://doi.org/10.1145/380752.380808},
    doi = {10.1145/380752.380808},
    booktitle = {Proceedings of the Thirty-Third Annual ACM Symposium on Theory of Computing},
    pages = {247–257},
    numpages = {11},
    location = {Hersonissos, Greece},
    series = {STOC '01}
}

@article{maaten2008tsne,
  author  = {Laurens van der Maaten and Geoffrey Hinton},
  title   = {Visualizing Data using t-SNE},
  journal = {Journal of Machine Learning Research},
  year    = {2008},
  volume  = {9},
  number  = {86},
  pages   = {2579--2605},
  url     = {http://jmlr.org/papers/v9/vandermaaten08a.html}
}

@article{aamari2019estimating,
    author = {Eddie Aamari and Jisu Kim and Fr{\'e}d{\'e}ric Chazal and Bertrand Michel and Alessandro Rinaldo and Larry Wasserman},
    title = {{Estimating the reach of a manifold}},
    volume = {13},
    journal = {Electronic Journal of Statistics},
    number = {1},
    publisher = {Institute of Mathematical Statistics and Bernoulli Society},
    pages = {1359 -- 1399},
    year = {2019},
    doi = {10.1214/19-EJS1551},
    URL = {https://doi.org/10.1214/19-EJS1551}
}

@article{boissonnat2023reach,
  title = {The Reach of Subsets of Manifolds},
  author = {Boissonnat, Jean-Daniel and Wintraecken, Mathijs},
  year = 2023,
  month = sep,
  journal = {Journal of Applied and Computational Topology},
  volume = {7},
  number = {3},
  pages = {619--641},
  issn = {2367-1734},
  doi = {10.1007/s41468-023-00116-x}
}

@article{niyogiFindingHomologySubmanifolds2008a,
  title = {Finding the {{Homology}} of {{Submanifolds}} with {{High Confidence}} from~{{Random~Samples}}},
  author = {Niyogi, Partha and Smale, Stephen and Weinberger, Shmuel},
  year = 2008,
  month = mar,
  journal = {Discrete \& Computational Geometry},
  volume = {39},
  number = {1},
  pages = {419--441},
  issn = {1432-0444},
  doi = {10.1007/s00454-008-9053-2}
}

@book{penrose2003random,
  title = {Random Geometric Graphs},
  author = {Penrose, Mathew},
  year = 2003,
  month = may,
  publisher = {Oxford University Press},
  doi = {10.1093/acprof:oso/9780198506263.001.0001},
  isbn = {978-0-19-850626-3}
}

@article{balister2005connectivity,
  title={Connectivity of random k-nearest-neighbour graphs}, 
  volume={37}, 
  DOI={10.1239/aap/1113402397}, 
  number={1}, 
  journal={Advances in Applied Probability}, 
  author={Balister, Paul and Bollobás, Béla and Sarkar, Amites and Walters, Mark}, 
  year={2005}, 
  pages={1–24}
}

@article{luxburg2007tutorial,
  title = {A Tutorial on Spectral Clustering},
  author = {{von Luxburg}, Ulrike},
  year = 2007,
  month = dec,
  journal = {Statistics and Computing},
  volume = {17},
  number = {4},
  pages = {395--416},
  issn = {1573-1375},
  doi = {10.1007/s11222-007-9033-z}
}

@article{tenenbaum2000isomap,
    author = {Joshua B. Tenenbaum  and Vin de Silva  and John C. Langford },
    title = {A Global Geometric Framework for Nonlinear Dimensionality Reduction},
    journal = {Science},
    volume = {290},
    number = {5500},
    pages = {2319-2323},
    year = {2000},
    doi = {10.1126/science.290.5500.2319},
    URL = {https://www.science.org/doi/abs/10.1126/science.290.5500.2319},
    eprint = {https://www.science.org/doi/pdf/10.1126/science.290.5500.2319}
}

@article{roweis2000lle,
    author = {Sam T. Roweis  and Lawrence K. Saul },
    title = {Nonlinear Dimensionality Reduction by Locally Linear Embedding},
    journal = {Science},
    volume = {290},
    number = {5500},
    pages = {2323-2326},
    year = {2000},
    doi = {10.1126/science.290.5500.2323},
    URL = {https://www.science.org/doi/abs/10.1126/science.290.5500.2323},
    eprint = {https://www.science.org/doi/pdf/10.1126/science.290.5500.2323}
}

@article{belkin2003laplacian,
    author = {Belkin, Mikhail and Niyogi, Partha},
    title = {Laplacian Eigenmaps for dimensionality reduction and data representation},
    year = {2003},
    issue_date = {June 2003},
    publisher = {MIT Press},
    address = {Cambridge, MA, USA},
    volume = {15},
    number = {6},
    issn = {0899-7667},
    url = {https://doi.org/10.1162/089976603321780317},
    doi = {10.1162/089976603321780317},
    journal = {Neural Comput.},
    month = jun,
    pages = {1373–1396},
    numpages = {24}
}

@misc{mcinnes2018umap,
  title={UMAP: Uniform Manifold Approximation and Projection for Dimension Reduction}, 
  author={Leland McInnes and John Healy and James Melville},
  year={2020},
  eprint={1802.03426},
  archivePrefix={arXiv},
  primaryClass={stat.ML},
  url={https://arxiv.org/abs/1802.03426}, 
}

@inproceedings{ester1996dbscan,
  author = {Ester, Martin and Kriegel, Hans-Peter and Sander, J{\"o}rg and Xu, Xiaowei},
    title = {A density-based algorithm for discovering clusters in large spatial databases with noise},
    year = {1996},
    publisher = {AAAI Press},
    booktitle = {Proceedings of the Second International Conference on Knowledge Discovery and Data Mining},
    pages = {226–231},
    numpages = {6},
    location = {Portland, Oregon},
    series = {KDD'96}
}

@article{ankerst1999optics,
    author = {Ankerst, Mihael and Breunig, Markus M. and Kriegel, Hans-Peter and Sander, J{\"o}rg},
    title = {OPTICS: ordering points to identify the clustering structure},
    year = {1999},
    journal = {ACM SIGMOD Record},
    issue_date = {June 1999},
    publisher = {Association for Computing Machinery},
    address = {New York, NY, USA},
    volume = {28},
    number = {2},
    issn = {0163-5808},
    url = {https://doi.org/10.1145/304181.304187},
    doi = {10.1145/304181.304187},
    journal = {SIGMOD Rec.},
    month = jun,
    pages = {49–60},
    numpages = {12}
}

@article{campello2015hdbscan,
  author = {Campello, Ricardo J. G. B. and Moulavi, Davoud and Zimek, Arthur and Sander, J{\"o}rg},
title = {Hierarchical Density Estimates for Data Clustering, Visualization, and Outlier Detection},
year = {2015},
issue_date = {July 2015},
publisher = {Association for Computing Machinery},
address = {New York, NY, USA},
volume = {10},
number = {1},
issn = {1556-4681},
url = {https://doi.org/10.1145/2733381},
doi = {10.1145/2733381},
journal = {ACM Trans. Knowl. Discov. Data},
month = jul,
articleno = {5},
numpages = {51}
}

@inproceedings{campello2013hdbscan,
  author="Campello, Ricardo J. G. B.
    and Moulavi, Davoud
    and Sander, Joerg",
    editor="Pei, Jian
    and Tseng, Vincent S.
    and Cao, Longbing
    and Motoda, Hiroshi
    and Xu, Guandong",
    title="Density-Based Clustering Based on Hierarchical Density Estimates",
    booktitle="Advances in Knowledge Discovery and Data Mining",
    year="2013",
    publisher="Springer Berlin Heidelberg",
    address="Berlin, Heidelberg",
    pages="160--172",
    isbn="978-3-642-37456-2"
}

@book{boissonnat2018book,
  place={Cambridge}, 
  series={Cambridge Texts in Applied Mathematics}, 
  title={Geometric and Topological Inference}, 
  publisher={Cambridge University Press}, 
  author={Boissonnat, Jean-Daniel and Chazal, Frédéric and Yvinec, Mariette}, year={2018}, 
  collection={Cambridge Texts in Applied Mathematics}
}

@article{federer1959curvature,
  title   = {Curvature Measures},
  author  = {Federer, Herbert},
  journal = {Transactions of the American Mathematical Society},
  volume  = {93},
  number  = {3},
  pages   = {418--491},
  year    = {1959},
  doi     = {10.1090/S0002-9947-1959-0110078-1}
}

@article{kleinberg2002impossibility,
  author = {Kleinberg, Jon},
 journal = {Advances in Neural Information Processing Systems},
 editor = {S. Becker and S. Thrun and K. Obermayer},
 pages = {},
 publisher = {MIT Press},
 title = {An Impossibility Theorem for Clustering},
 url = {https://proceedings.neurips.cc/paper_files/paper/2002/file/43e4e6a6f341e00671e123714de019a8-Paper.pdf},
 volume = {15},
 year = {2002}
}

@article{tibshirani2001estimating,
  title = {Estimating the Number of Clusters in a Data Set via the Gap Statistic},
  author = {Tibshirani, Robert and Walther, Guenther and Hastie, Trevor},
  year = 2001,
  month = jul,
  journal = {Journal of the Royal Statistical Society Series B: Statistical Methodology},
  volume = {63},
  number = {2},
  eprint = {https://academic.oup.com/jrsssb/article-pdf/63/2/411/49590410/jrsssb\_63\_2\_411.pdf},
  pages = {411--423},
  issn = {1369-7412},
  doi = {10.1111/1467-9868.00293}
}

@article{hennig2015true,
  title = {What are the true clusters?},
    journal = {Pattern Recognition Letters},
    volume = {64},
    pages = {53-62},
    year = {2015},
    note = {Philosophical Aspects of Pattern Recognition},
    issn = {0167-8655},
    doi = {https://doi.org/10.1016/j.patrec.2015.04.009},
    url = {https://www.sciencedirect.com/science/article/pii/S0167865515001269},
    author = {Christian Hennig}
}

@article{eisen1998cluster,
  author = {Michael B. Eisen  and Paul T. Spellman  and Patrick O. Brown  and David Botstein },
    title = {Cluster analysis and display of genome-wide expression patterns},
    journal = {Proceedings of the National Academy of Sciences},
    volume = {95},
    number = {25},
    pages = {14863-14868},
    year = {1998},
    doi = {10.1073/pnas.95.25.14863},
    URL = {https://www.pnas.org/doi/abs/10.1073/pnas.95.25.14863},
    eprint = {https://www.pnas.org/doi/pdf/10.1073/pnas.95.25.14863}
}

@article{chari2023specious,
  doi = {10.1371/journal.pcbi.1011288},
    author = {Chari, Tara AND Pachter, Lior},
    journal = {PLOS Computational Biology},
    publisher = {Public Library of Science},
    title = {The specious art of single-cell genomics},
    year = {2023},
    month = {08},
    volume = {19},
    url = {https://doi.org/10.1371/journal.pcbi.1011288},
    pages = {1-20},
    number = {8},
}

@article{dalmaijer2022statistical,
  title = {Statistical Power for Cluster Analysis},
  author = {Dalmaijer, Edwin S. and Nord, Camilla L. and Astle, Duncan E.},
  year = 2022,
  month = may,
  journal = {BMC Bioinformatics},
  volume = {23},
  number = {1},
  pages = {205},
  issn = {1471-2105},
  doi = {10.1186/s12859-022-04675-1}
}

@article{button2013power,
  title = {Power Failure: Why Small Sample Size Undermines the Reliability of Neuroscience},
  author = {Button, Katherine S. and Ioannidis, John P. A. and Mokrysz, Claire and Nosek, Brian A. and Flint, Jonathan and Robinson, Emma S. J. and Munaf{\`o}, Marcus R.},
  year = 2013,
  month = may,
  journal = {Nature Reviews Neuroscience},
  volume = {14},
  number = {5},
  pages = {365--376},
  issn = {1471-0048},
  doi = {10.1038/nrn3475}
}

@article{dyballa_population_2024,
	author = {Luciano Dyballa  and Andra M. Rudzite  and Mahmood S. Hoseini  and Mishek Thapa  and Michael P. Stryker  and Greg D. Field  and Steven W. Zucker },
    title = {Population encoding of stimulus features along the visual hierarchy},
    journal = {Proceedings of the National Academy of Sciences},
    volume = {121},
    number = {4},
    pages = {e2317773121},
    year = {2024},
    doi = {10.1073/pnas.2317773121},
    URL = {https://www.pnas.org/doi/abs/10.1073/pnas.2317773121},
    eprint = {https://www.pnas.org/doi/pdf/10.1073/pnas.2317773121}
}

@article{dyballa2024functional,
  author = {Dyballa, Luciano and Field, Greg D. and Stryker, Michael P. and Zucker, Steven W.},
	title = {Functional organization and natural scene responses across mouse visual cortical areas revealed with encoding manifolds},
	elocation-id = {2024.10.24.620089},
	year = {2025},
	doi = {10.1101/2024.10.24.620089},
	publisher = {Cold Spring Harbor Laboratory},
	URL = {https://www.biorxiv.org/content/early/2025/10/23/2024.10.24.620089},
	eprint = {https://www.biorxiv.org/content/early/2025/10/23/2024.10.24.620089.full.pdf},
	journal = {bioRxiv}
}

@inproceedings{Miron1996BIRCH,
author = {Zhang, Tian and Ramakrishnan, Raghu and Livny, Miron},
title = {BIRCH: an efficient data clustering method for very large databases},
year = {1996},
isbn = {0897917944},
publisher = {Association for Computing Machinery},
address = {New York, NY, USA},
url = {https://doi.org/10.1145/233269.233324},
doi = {10.1145/233269.233324},
booktitle = {Proceedings of the 1996 ACM SIGMOD International Conference on Management of Data},
pages = {103–114},
numpages = {12},
location = {Montreal, Quebec, Canada},
series = {SIGMOD '96}
}

@article{maier2009mutualknn,
author = {Maier, Markus and Hein, Matthias and von Luxburg, Ulrike},
title = {Optimal construction of k-nearest-neighbor graphs for identifying noisy clusters},
year = {2009},
issue_date = {April, 2009},
publisher = {Elsevier Science Publishers Ltd.},
address = {GBR},
volume = {410},
number = {19},
issn = {0304-3975},
url = {https://doi.org/10.1016/j.tcs.2009.01.009},
doi = {10.1016/j.tcs.2009.01.009},
journal = {Theor. Comput. Sci.},
month = apr,
pages = {1749–1764},
numpages = {16}
}

@article{brito1997connectivity,
    title = {Connectivity of the mutual k-nearest-neighbor graph in clustering and outlier detection},
    journal = {Statistics \& Probability Letters},
    volume = {35},
    number = {1},
    pages = {33-42},
    year = {1997},
    issn = {0167-7152},
    doi = {https://doi.org/10.1016/S0167-7152(96)00213-1},
    url = {https://www.sciencedirect.com/science/article/pii/S0167715296002131},
    author = {M.R. Brito and E.L. Chávez and A.J. Quiroz and J.E. Yukich}
}

@article{blei2006variational,
    author = {David M. Blei and Michael I. Jordan},
    title = {{Variational inference for Dirichlet process mixtures}},
    volume = {1},
    journal = {Bayesian Analysis},
    number = {1},
    publisher = {International Society for Bayesian Analysis},
    pages = {121 -- 143},
    year = {2006},
    doi = {10.1214/06-BA104},
    URL = {https://doi.org/10.1214/06-BA104}
}

@article{carlsson2010characterization,
  author  = {Gunnar Carlsson and Facundo M{{\'e}}moli},
  title   = {Characterization, Stability and Convergence of Hierarchical Clustering Methods},
  journal = {Journal of Machine Learning Research},
  year    = {2010},
  volume  = {11},
  number  = {47},
  pages   = {1425-1470},
  url     = {http://jmlr.org/papers/v11/carlsson10a.html}
}

@article{still2004howmanyclusters,
    author = {Still, Susanne and Bialek, William},
    title = {How Many Clusters? An Information-Theoretic Perspective},
    year = {2004},
    issue_date = {December 2004},
    publisher = {MIT Press},
    address = {Cambridge, MA, USA},
    volume = {16},
    number = {12},
    issn = {0899-7667},
    url = {https://doi.org/10.1162/0899766042321751},
    doi = {10.1162/0899766042321751},
    journal = {Neural Comput.},
    month = dec,
    pages = {2483–2506},
    numpages = {24}
}

@article{wade2018bayesian,
  title   = {Bayesian Cluster Analysis: Point Estimation and Credible Balls (with Discussion)},
  author  = {Wade, Sara and Ghahramani, Zoubin},
  journal = {Bayesian Analysis},
  volume  = {13},
  number  = {2},
  pages   = {559--626},
  year    = {2018},
  doi     = {10.1214/17-BA1073}
}

\appendix

\noindent\paragraph{Appendix overview.} 
\textit{Theory} (App.~\ref{app:knn-threshold}) proves Theorem~\ref{thm:knn-threshold}, and the observed-sample-offset lemma. 
\textit{Algorithm} (Apps.~\ref{app:impl}--\ref{app:full-algorithm}) describes the per-edge gate, minimum-component size, mass-coverage alternative, practical estimate $\widehat K_{\mathrm{prac}}$, baseline parameter grids, and full pseudocode.
\textit{Experiments} (Apps.~\ref{app:multi-seed}--\ref{app:neuro-behavior}) reports inter-seed stability, the $(\delta,\alpha)$ sensitivity grid, the perturbation walk, the sampling-size sweep, the synthetic-suite catalog and per-family aggregate, the mass-coverage bracket sensitivity, and the neuroscience filtration figure.



\section{Proof of Theorem~\ref{thm:knn-threshold}}
\label{app:knn-threshold}

We now prove the fill-normalized separation threshold for the mutual-$k$NN graph stated in the main text. We first control the fill distance and the largest $k$NN radius at the common scale $(\log n/n)^{1/d}$, then compare that scale with the manifold offset: a sufficiently large offset excludes cross-component listings, while a sufficiently small offset produces a short cross pair whose endpoints must list one another.

\subsection{Notation and assumptions}
\label{app:mutual-knn-setup}

\paragraph{Notation reference.}
We use the following notation throughout the appendix.
\begin{center}
\begingroup
\renewcommand{\arraystretch}{1.08}
\begin{tabular}{@{}p{0.24\linewidth}p{0.67\linewidth}@{}}
$\mathcal M_i,\mu_i$ & manifold component and sampling measure, $i\in\{1,2\}$;\\
$\underline c,\overline c,r_*,d$ & constants in the local mass bound
$\mu_i(B(x,r))\asymp r^d$;\\
$S_i,n_i,n,n_{\min},n_{\max}$ & labeled samples and their sizes;\\
$\Delta,h_i,h_{\max}$ & manifold offset, component fill distances, and their maximum;\\
$N_k(z),D_k(z)$ & ordered $k$NN list of $z$ and its $k$th-neighbor radius;\\
$G_{\mathrm{mut}}$ & mutual-$k$NN graph on the full sample;\\
$A,\delta,\varepsilon$ & graph-scale coefficient, target failure level, and concentration slack in the neighborhood-radius bound;\\
$L,k,R_n,r_k$ & logarithmic factor, graph degree, sample-size log ratio, and uniform neighborhood-radius scale;\\
$M_{\mathrm{high}},M_{\mathrm{low}}$ & fill-normalized no-edge and bridge thresholds.
\end{tabular}
\endgroup
\end{center}

\begin{assumption}[Geometry and local mass]
\label{ass:mutual-knn-geometry}
Let $d\ge1$.  The sets $\mathcal M_1,\mathcal M_2\subset\mathbb R^D$ are compact, connected, disjoint, $d$-dimensional $C^2$ submanifolds with positive reach.  For each $i\in\{1,2\}$, let $\mu_i$ be a Borel probability measure supported on $\mathcal M_i$.  There are constants $0<\underline c\le\overline c<\infty$ and $r_*>0$ such that
\begin{equation}
\label{eq:mutual-knn-local-mass}
    \underline c\,r^d
    \le \mu_i\bigl(B(x,r)\bigr)
    \le \overline c\,r^d,
    \qquad
    x\in\mathcal M_i,
    \quad 0<r\le r_*.
\end{equation}
Here $B(x,r)$ is the closed ambient Euclidean ball.
\end{assumption}

\begin{remark}[How the geometry enters]
\label{rem:mutual-knn-geometry-role}
The manifold assumptions enter only through \eqref{eq:mutual-knn-local-mass}: curvature is absorbed into the constants and the radius $r_*$ below which Euclidean balls have mass comparable to $r^d$.
\end{remark}

\begin{assumption}[Sampling]
\label{ass:mutual-knn-sampling}
For $i\in\{1,2\}$, let
\begin{equation}
\label{eq:mutual-knn-samples}
    S_i:=\{X_{i1},\ldots,X_{in_i}\},
    \qquad
    X_{i1},\ldots,X_{in_i}\stackrel{\mathrm{i.i.d.}}{\sim}\mu_i,
\end{equation}
and assume that $S_1$ and $S_2$ are independent. Write
\begin{equation}
\label{eq:mutual-knn-sample-sizes}
    S:=S_1\cup S_2,
    \qquad n:=n_1+n_2,
    \qquad n_{\min}:=\min\{n_1,n_2\},
    \qquad n_{\max}:=\max\{n_1,n_2\}.
\end{equation}
For the mutual-cross-edge direction only, assume
\begin{equation}
\label{eq:mutual-knn-balance}
    n_{\max}\le2n_{\min}.
\end{equation}
This sample balance condition prevents one component from being sampled at a much higher resolution than the other.
\end{assumption}

We write the manifold offset and the component-wise fill distances as
\begin{align}
\label{eq:mutual-knn-offset}
    \Delta
    &:=\min_{x\in\mathcal M_1,\,y\in\mathcal M_2}\|x-y\|,
    \\
\label{eq:mutual-knn-fill}
    h_i
    &:=\sup_{x\in\mathcal M_i}\min_{z\in S_i}\|x-z\|,
    &
    h_{\max}&:=\max\{h_1,h_2\}.
\end{align}
Compactness ensures that the minimum defining $\Delta$ is attained, while disjointness ensures that $\Delta>0$.  The quantity $h_i$ is the radius of the largest region of $\mathcal M_i$ left uncovered by $S_i$, and therefore measures the realized sampling resolution on the $i$th component.

\noindent To make the graph well defined when distances tie, assign every observation a fixed label before sampling and order candidate neighbors first by distance and then by this label. Let $N_k(z)$ be the first $k$ points in that order, and let $D_k(z)$ be the distance to the $k$th point. The mutual-$k$NN graph is
\begin{equation}
\label{eq:mutual-knn-graph}
    G_{\mathrm{mut}}=(S,E_{\mathrm{mut}}),
    \qquad
    E_{\mathrm{mut}}:=\bigl\{\{z,w\}:w\in N_k(z)\ \text{and}\ z\in N_k(w)\bigr\}.
\end{equation}
A \emph{cross edge} has one endpoint in each component; let $\mathsf{Cross}_{\mathrm{mut}}$ denote the event that such an edge exists. For $z\in S_i$ and $t>0$, define the same-component count
\begin{equation}
\label{eq:mutual-knn-local-count}
    Q(z;t)
    :=\#\bigl\{w\in S_i\setminus\{z\}:\|w-z\|\le t\bigr\}.
\end{equation}
Two immediate consequences will be used throughout the proof:
\begin{align}
\label{eq:mutual-knn-count-radius}
    Q(z;t)\ge k
    &\quad\Longrightarrow\quad D_k(z)\le t,
    \\
\label{eq:mutual-knn-outside-list}
    \|w-z\|>D_k(z)
    &\quad\Longrightarrow\quad w\notin N_k(z).
\end{align}
The first bounds a neighborhood radius using same-component points alone, while the second says that a point beyond the $k$th-neighbor radius cannot appear in the neighbor list.

\begin{assumption}[Parameters and scale conditions]
\label{ass:mutual-knn-parameters}
Fix $\delta,\varepsilon\in(0,1)$ and define
\begin{equation}
\label{eq:mutual-knn-parameters}
    L:=\log\frac{4n}{\delta},
    \qquad
    k:=\lceil AL\rceil,
    \qquad
    A\ge\frac{3}{\varepsilon^2},
    \qquad
    R_n:=\frac{\log(4n/\delta)}{\log(n_{\min}/\delta)}.
\end{equation}
For the mutual-cross-edge direction, fix $a\in(0,1/8)$ and set $B:=1+2a$.  The proof uses only $a>0$ and $B\Delta\le r_*$. Now set
\begin{equation}
\label{eq:mutual-knn-rk}
    r_k:=\left(
        \frac{2k}{(1-\varepsilon)\underline c\,n_{\min}}
    \right)^{1/d}.
\end{equation}
We assume
\begin{align}
\label{eq:mutual-knn-scale-one}
    k&\le\frac{n_{\min}}2,
    &
    r_k&\le r_*,
    \\
\label{eq:mutual-knn-scale-two}
    \frac{2^d\log(n_i/\delta)}{\underline c\,n_i}
    &\le r_*^d,
    && i\in\{1,2\},
    \\
\label{eq:mutual-knn-scale-three}
    n_i
    &\ge\max\{2^{d+2},\log(1/\delta)\},
    &
    \sqrt{n_i}
    &\ge
    \frac{2^{d+1}\log(n_i/\delta)}{\delta^{3/2}},
    \qquad i\in\{1,2\}.
\end{align}
For the lower-threshold direction, we additionally assume
\begin{equation}
\label{eq:mutual-knn-local-gap}
    B\Delta\le r_*.
\end{equation}
These conditions ensure that every radius used below remains inside the local mass condition and that the concentration bounds hold with the stated probabilities. The inequalities in \eqref{eq:mutual-knn-scale-two}--\eqref{eq:mutual-knn-scale-three} are explicit sample-size requirements; the balance and local-gap conditions remain separate structural assumptions.
\end{assumption}

Since $n\ge4$, $L>2$, and therefore
\begin{equation}
\label{eq:mutual-knn-ceiling}
    AL\ge2,
    \qquad
    AL\le k\le AL+1.
\end{equation}
We use the standard binomial bounds
\begin{align}
\label{eq:mutual-knn-chernoff-lower}
    X\sim\operatorname{Bin}(q,p),\ \mu=qp,
    \quad t\in(0,1)
    &\Longrightarrow
    \mathbb P\bigl(X\le(1-t)\mu\bigr)
    \le e^{-t^2\mu/2},
    \\
\label{eq:mutual-knn-chernoff-upper}
    X\sim\operatorname{Bin}(q,p),\ \mu=qp,
    \quad s\ge\mu
    &\Longrightarrow
    \mathbb P(X\ge s)
    \le\left(\frac{e\mu}{s}\right)^s.
\end{align}

\noindent We can now state the threshold result proved in the remainder of the appendix.

\begin{theorem}[Fill-normalized threshold for the mutual-$k$NN graph]
\label{thm:mutual-knn-threshold}
Under Assumptions~\ref{ass:mutual-knn-geometry}--\ref{ass:mutual-knn-parameters}, define
\begin{align}
\label{eq:mutual-knn-thresholds}
    M_{\mathrm{high}}
    &:=\left(\frac{12AR_n\overline c}{(1-\varepsilon)\underline c}
    \right)^{1/d},
    &
    M_{\mathrm{low}}
    &:=\left(\frac{AR_n\underline c}{2^{d+4}\overline c B^d}
    \right)^{1/d},
    \\
\label{eq:mutual-knn-gamma}
    \gamma
    &:=\log\frac4e-\frac1A
    \ge\frac1{20}.
\end{align}
Then:
\begin{enumerate}
\item[(i)] Without the balance condition \eqref{eq:mutual-knn-balance},
\begin{equation}
\label{eq:mutual-knn-main-upper}
    \mathbb P\left(
        \{\Delta>M_{\mathrm{high}}h_{\max}\}
        \cap\mathsf{Cross}_{\mathrm{mut}}
    \right)
    \le\frac{3\delta}{8}.
\end{equation}

\item[(ii)] Under \eqref{eq:mutual-knn-balance} and
\eqref{eq:mutual-knn-local-gap},
\begin{equation}
\label{eq:mutual-knn-main-lower}
    \mathbb P\left(
        \{\Delta<M_{\mathrm{low}}h_{\max}\}
        \cap\mathsf{Cross}_{\mathrm{mut}}^{\,c}
    \right)
    \le
    \frac{\delta}{4}
    +2e^{-\underline c a^dn_{\min}\Delta^d}
    +e^{-\gamma k}.
\end{equation}
\end{enumerate}
\end{theorem}

For the remainder of the two-component proof, abbreviate
\begin{equation}
\label{eq:mutual-knn-proof-abbreviations}
    m:=n_{\min},
    \qquad
    h:=h_{\max},
    \qquad
    \ell:=\log\frac{m}{\delta},
    \qquad
    s:=\left(\frac{\ell}{m}\right)^{1/d},
    \qquad
    r:=r_k.
\end{equation}
Thus $R_n=L/\ell$. The proof follows the four implications
\begin{equation*}
\begin{array}{rcl}
\text{covering and packing}
&\Longrightarrow&
 c_{\mathrm{fill}}s\le h\le C_{\mathrm{fill}}s,\\[1mm]
\text{same-component counts}
&\Longrightarrow&
 \max_{z\in S}D_k(z)\le r,\\[1mm]
 h\ge c_{\mathrm{fill}}s,\ \max_zD_k(z)\le r
&\Longrightarrow&
 \text{no cross edge above }M_{\mathrm{high}},\\[1mm]
 h\le C_{\mathrm{fill}}s,\ \text{cap occupancy and low crowding}
&\Longrightarrow&
 \text{a cross edge below }M_{\mathrm{low}}.
\end{array}
\end{equation*}

\subsection{The neighborhood scale}
\label{app:mutual-knn-scale}

We begin by identifying the scale at which the fill distance, determined by the samples, can be approximated. The fill distance is the radius of the largest hole left by the sample, and we now bound it as a function of the number of observations and the local geometric constants.

\begin{lemma}[Covering and packing from local mass]
\label{lem:mutual-knn-cover-pack}
Fix $i\in\{1,2\}$.
\begin{enumerate}
\item[(a)] For $0<r_0\le r_*$, there exist $x_1,\ldots,x_N\in\mathcal M_i$ such that
\begin{equation}
\label{eq:mutual-knn-cover-number}
    \mathcal M_i\subseteq\bigcup_{j=1}^N B(x_j,r_0),
    \qquad
    N\le\frac{2^d}{\underline c\,r_0^d}.
\end{equation}
\item[(b)] For $0<2r_0\le r_*$, there exist $y_1,\ldots,y_M\in\mathcal M_i$ such that the balls $B(y_j,r_0)$ are pairwise disjoint and
\begin{equation}
\label{eq:mutual-knn-packing-number}
    M\ge\frac{1}{2^d\overline c\,r_0^d}.
\end{equation}
\end{enumerate}
\end{lemma}

\begin{proof}
For the covering statement, choose a maximal set $\{x_1,\ldots,x_N\}\subset\mathcal M_i$ whose distinct points are more than $r_0$ apart. Here maximal means that no further point can be added while preserving this separation, and therefore maximality gives the radius-$r_0$ cover. The balls $B(x_j,r_0/2)$ are pairwise disjoint, and hence
\begin{align*}
    1
    &\ge\sum_{j=1}^N\mu_i(B(x_j,r_0/2))
    \ge N\underline c(r_0/2)^d
    =N\frac{\underline c r_0^d}{2^d},
\end{align*}
which rearranges to \eqref{eq:mutual-knn-cover-number}.

For the packing statement, choose instead a maximal set $\{y_1,\ldots,y_M\}$ whose distinct points are more than $2r_0$ apart. The radius-$r_0$ balls are now pairwise disjoint, while maximality says that the radius-$2r_0$ balls cover $\mathcal M_i$. Applying the upper mass bound to this cover gives
\begin{align*}
    1
    &=\mu_i(\mathcal M_i)
    \le\sum_{j=1}^M\mu_i(B(y_j,2r_0))
    \le M\overline c(2r_0)^d,
\end{align*}
which rearranges to \eqref{eq:mutual-knn-packing-number}.
\end{proof}

\begin{lemma}[Two-sided fill-distance control]
\label{lem:mutual-knn-fill}
Define
\begin{equation}
\label{eq:mutual-knn-fill-constants}
    C_{\mathrm{fill}}
    :=\left(\frac{2^{d+1}}{\underline c}\right)^{1/d},
    \qquad
    c_{\mathrm{fill}}
    :=\left(\frac{1}{4\overline c}\right)^{1/d}.
\end{equation}
With probability at least $1-\delta/2$, simultaneously for $i\in\{1,2\}$,
\begin{equation}
\label{eq:mutual-knn-fill-two-sided}
    c_{\mathrm{fill}}
    \left(\frac{\log(n_i/\delta)}{n_i}\right)^{1/d}
    \le h_i\le
    C_{\mathrm{fill}}
    \left(\frac{\log(n_i/\delta)}{n_i}\right)^{1/d}.
\end{equation}
For each fixed component, each one-sided inequality fails with probability at most $\delta/8$.
\end{lemma}

\begin{proof}
Fix one component $i$ and, only in this proof, abbreviate $N=n_i$, $H=h_i$, and $\mu=\mu_i$. We treat the upper and lower bounds separately because they encode opposite geometric statements: the upper bound rules out an unusually large uncovered region, whereas the lower bound shows that a hole of the canonical sampling scale remains.

\paragraph{Upper bound.}
Take the radius-$u$ cover from Lemma~\ref{lem:mutual-knn-cover-pack}. If every cover ball contains a sample, the entire manifold lies within distance $2u$ of the sample. Indeed, for any $x\in\mathcal M_i$, choose a center $x_j$ with $\|x-x_j\|\le u$ and a sample $z\in S_i\cap B(x_j,u)$. Then
$$
    \|x-z\|\le\|x-x_j\|+\|x_j-z\|\le2u.
$$
Thus $H\le2u$. Equivalently, if $H>2u$, then at least one covering ball is empty. This converts the global event that the sample leaves a large hole into a finite union of occupancy failures.

For a fixed cover ball, write $p_j=\mu(B(x_j,u))$. Since $p_j\ge\underline c u^d$, independence of the $N$ draws gives
\begin{align*}
    \mathbb P(S_i\cap B(x_j,u)=\varnothing)
    &=(1-p_j)^N
    \le e^{-Np_j}
    \le e^{-N\underline c u^d}.
\end{align*}
Summing these probabilities over the cover yields
\begin{align}
\label{eq:mutual-knn-fill-upper-generic}
    \mathbb P(H>2u)
    &\le
    \frac{2^d}{\underline c u^d}
    e^{-N\underline c u^d}.
    &&\text{cover size $\times$ one-ball failure}
\end{align}
We now choose $u$ so that the exponential term dominates the number of covering cells. Set
\begin{equation}
\label{eq:mutual-knn-r-upper}
    u_+^d:=\frac{2\log(N/\delta)}{\underline c N}.
\end{equation}
At scale $u$, the cover has order $u^{-d}$ cells, while a fixed cell is empty with probability at most order $e^{-Nu^d}$. The choice \eqref{eq:mutual-knn-r-upper} balances these terms by making $N\underline c u_+^d=2\log(N/\delta)$. The scale assumptions ensure that $u_+\le r_*$, and at this radius
\begin{align*}
    e^{-N\underline c u_+^d}
    &=e^{-2\log(N/\delta)}
    =\left(\frac{\delta}{N}\right)^2,
    \\
    \frac{2^d}{\underline c u_+^d}
    &=\frac{2^{d-1}N}{\log(N/\delta)}.
\end{align*}
Substituting both identities into \eqref{eq:mutual-knn-fill-upper-generic},
\begin{align}
\label{eq:mutual-knn-fill-upper-prob}
    \mathbb P(H>2u_+)
    &\le
    \frac{2^{d-1}N}{\log(N/\delta)}
    \left(\frac{\delta}{N}\right)^2
    \notag\\
    &=
    \frac{2^{d-1}\delta^2}{N\log(N/\delta)}
    \le\frac{\delta}{8}.
    &&\text{by \eqref{eq:mutual-knn-scale-three}}
\end{align}
The last inequality is equivalent to $2^{d+2}\delta\le N\log(N/\delta)$, which follows from $N\ge2^{d+2}$, $\delta<1$, and $\log(N/\delta)>1$. We have therefore shown that, except with probability $\delta/8$,
\begin{equation}
\label{eq:mutual-knn-fill-upper-final}
    H\le2u_+
    =C_{\mathrm{fill}}
    \left(\frac{\log(N/\delta)}{N}\right)^{1/d}.
\end{equation}

\paragraph{Lower bound.}
For the lower bound the logic is reversed. We take the disjoint radius-$u$ packing balls and count how many receive no sample. Define
\begin{equation}
\label{eq:mutual-knn-empty-count}
    Y_j:=\mathbf 1\{S_i\cap B(y_j,u)=\varnothing\},
    \qquad
    Z:=\sum_{j=1}^M Y_j.
\end{equation}
If $Z\ge1$, the center of an empty packed ball is at distance at least $u$ from $S_i$, and therefore $H\ge u$. It is enough to choose $u$ so that $Z$ is positive with high probability. We take
\begin{equation}
\label{eq:mutual-knn-r-lower}
    u_-^d:=\frac{\log(N/\delta)}{4\overline c N}.
\end{equation}
The scale assumptions give $2u_-\le r_*$ and $\overline c u_-^d\le1/2$. If $p_j:=\mu(B(y_j,u_-))$, then $p_j\le\overline c u_-^d$ and
\begin{align}
\label{eq:mutual-knn-one-empty-prob}
    \mathbb E Y_j
    &=(1-p_j)^N
    \ge(1-\overline c u_-^d)^N
    \ge e^{-2N\overline c u_-^d}
    =\left(\frac{\delta}{N}\right)^{1/2}.
    &&\text{since }1-v\ge e^{-2v}
\end{align}
The packing contains at least
\begin{equation*}
    M
    \ge\frac{1}{2^d\overline c u_-^d}
    =\frac{2^{2-d}N}{\log(N/\delta)}
\end{equation*}
balls. Consequently,
\begin{align}
\label{eq:mutual-knn-empty-mean}
    \mathbb E Z
    &\ge
    M\left(\frac{\delta}{N}\right)^{1/2}
    \ge
    \frac{2^{2-d}\sqrt\delta\sqrt N}{\log(N/\delta)}
    \ge\frac8\delta.
    &&\text{by \eqref{eq:mutual-knn-scale-three}}
\end{align}
More explicitly,
\begin{equation*}
    \frac{2^{2-d}\sqrt\delta\sqrt N}{\log(N/\delta)}
    \ge
    \frac{2^{2-d}\sqrt\delta}{\log(N/\delta)}
    \frac{2^{d+1}\log(N/\delta)}{\delta^{3/2}}
    =\frac8\delta.
\end{equation*}
This is the only place where the stronger square-root condition in \eqref{eq:mutual-knn-scale-three} is used. The expected number of empty packed balls is now large, but expectation alone does not rule out $Z=0$, so we also control the variance.

For $j\ne q$, disjointness gives
\begin{align*}
    \mathbb E(Y_jY_q)
    &=(1-p_j-p_q)^N
    \le\bigl((1-p_j)(1-p_q)\bigr)^N
    =\mathbb E Y_j\,\mathbb E Y_q.
\end{align*}
The inequality follows from $1-p_j-p_q\le(1-p_j)(1-p_q)$. Thus emptiness of two disjoint cells is no more likely than it would be under independence, and $\operatorname{Cov}(Y_j,Y_q)\le0$. Hence
\begin{equation}
\label{eq:mutual-knn-empty-variance}
    \operatorname{Var}(Z)
    =\sum_j\operatorname{Var}(Y_j)
      +2\sum_{j<q}\operatorname{Cov}(Y_j,Y_q)
    \le\sum_j\operatorname{Var}(Y_j)
    \le\mathbb E Z.
\end{equation}
On the event $Z=0$, the deviation from the mean is exactly $\mathbb EZ$. Chebyshev's inequality therefore gives
\begin{align}
\label{eq:mutual-knn-empty-chebyshev}
    \mathbb P(Z=0)
    &\le
    \frac{\operatorname{Var}(Z)}{(\mathbb EZ)^2}
    \le\frac1{\mathbb EZ}
    \le\frac{\delta}{8}.
\end{align}
Outside this event, at least one packed ball is empty and hence $H\ge u_-$, which is the lower inequality in \eqref{eq:mutual-knn-fill-two-sided}. Applying the two one-sided estimates to both components and taking a union bound completes the proof.
\end{proof}

We now rewrite the component-wise bounds at the common scale $s=(\ell/m)^{1/d}$. Since $g(x):=\log(x/\delta)/x$ satisfies $g'(x)=(1-\log(x/\delta))/x^2\le0$ over the present sample-size range, $g$ is decreasing. Hence, whenever the upper fill bound holds for both components,
$$
    h^d
    =\max_{i\in\{1,2\}}h_i^d
    \le
    C_{\mathrm{fill}}^d\frac{\ell}{m}.
$$
Each component-wise upper bound fails with probability at most $\delta/8$, so
\begin{equation}
\label{eq:mutual-knn-fill-up-event}
    \mathcal E_{\mathrm{cover}}
    :=\{h\le C_{\mathrm{fill}}s\},
    \qquad
    \mathbb P(\mathcal E_{\mathrm{cover}}^c)\le\frac{\delta}{4}.
\end{equation}
Here $\mathcal E_{\mathrm{cover}}$ is the event that the a sample leaves no hole larger than the stated covering scale.

For the lower bound, choose $i_\star\in\{1,2\}$ with $n_{i_\star}=m$. Since $h\ge h_{i_\star}$, the lower fill-distance bound for this one component gives
$$
    h^d
    \ge h_{i_\star}^d
    \ge
    c_{\mathrm{fill}}^d\frac{\ell}{m}.
$$
Consequently,
\begin{equation}
\label{eq:mutual-knn-fill-low-event}
    \mathcal E_{\mathrm{hole}}
    :=\{h\ge c_{\mathrm{fill}}s\},
    \qquad
    \mathbb P(\mathcal E_{\mathrm{hole}}^c)\le\frac{\delta}{8}.
\end{equation}
The upper bound requires controlling of both components, while the lower bound requires only the least-sampled component.

\begin{lemma}[Uniform $k$NN-radius bound]
\label{lem:mutual-knn-radius}
Let
\begin{equation}
\label{eq:mutual-knn-radius-event}
    \mathcal E_{\mathrm{rad}}
    :=\{D_k(z)\le r\text{ for every }z\in S\}.
\end{equation}
Then
\begin{equation}
\label{eq:mutual-knn-radius-probability}
    \mathbb P(\mathcal E_{\mathrm{rad}}^c)
    \le\frac{\delta}{4}.
\end{equation}
That is, with probability at least $1-\delta/4$, every sample point has $k$ neighbors within distance $r$.
\end{lemma}

\begin{proof}
We first bound the failure probability for one labeled observation and then take a union bound over all $n$ labels. Fix $z=X_{ij}\in S_i$. Conditional on $z=x$, the remaining $n_i-1$ observations in the same component are independent with law $\mu_i$, so the count introduced in \eqref{eq:mutual-knn-local-count} satisfies
\begin{equation}
\label{eq:mutual-knn-conditional-count}
    Q(z;r)\mid\{z=x\}
    \sim
    \operatorname{Bin}\bigl(n_i-1,\mu_i(B(x,r))\bigr).
\end{equation}
Because $n_i-1\ge n_i/2\ge m/2$ and $r\le r_*$, the local lower mass bound gives, uniformly in $x$,
\begin{align}
\label{eq:mutual-knn-count-mean}
    \mathbb E[Q(z;r)\mid z=x]
    &=(n_i-1)\mu_i(B(x,r))
    \notag\\
    &\ge(n_i-1)\underline c\,r^d
    \ge\frac m2\underline c
        \frac{2k}{(1-\varepsilon)\underline c m}
    =\frac{k}{1-\varepsilon}.
\end{align}
Thus observing fewer than $k$ same-component points can be bounded by Chernoff, and so \eqref{eq:mutual-knn-chernoff-lower} conditionally on $z=x$ gives,
\begin{align}
\label{eq:mutual-knn-one-label-bound}
    \mathbb P(Q(z;r)<k\mid z=x)
    &\le
    \exp\left(
        -\frac{\varepsilon^2k}{2(1-\varepsilon)}
    \right)
    \le e^{-L}
    =\frac{\delta}{4n}.
\end{align}
The last inequality uses $k\ge AL$ and $A\ge3/\varepsilon^2$. Since the bound is uniform in $x$, averaging over $z$ gives the same unconditional inequality.

If $Q(z;r)\ge k$, then the full sample contains at least $k$ points within distance $r$ of $z$, and therefore $D_k(z)\le r$. Equivalently,
\begin{equation}
\label{eq:mutual-knn-radius-failure-inclusion}
    \{D_k(X_{ij})>r\}
    \subseteq
    \{Q(X_{ij};r)<k\}.
\end{equation}
It follows that $\mathbb P(D_k(X_{ij})>r)\le\delta/(4n)$. Taking a union bound over all observation labels gives
\begin{align*}
    \mathbb P(\mathcal E_{\mathrm{rad}}^c)
    &\le
    \sum_{i=1}^2\sum_{j=1}^{n_i}
    \mathbb P(D_k(X_{ij})>r)
    \le n\frac{\delta}{4n}
    =\frac{\delta}{4}.
\end{align*}
\end{proof}

\subsection{No cross-component edges}
\label{app:mutual-knn-no-cross}

We now compare the two scales established above. On the radius event, every point finds its $k$ nearest neighbors within $r$; on the hole event, the same deterministic radius can be expressed as a fixed multiple of the realized fill distance. Once this comparison is made, the no-cross-edge conclusion follows directly from the definition of the manifold offset.

\begin{proposition}[No mutual cross edge above the upper threshold]
\label{prop:mutual-knn-no-cross}
On
\begin{equation}
\label{eq:mutual-knn-upper-event}
    \mathcal E_{\mathrm{sep}}
    :=\mathcal E_{\mathrm{rad}}\cap\mathcal E_{\mathrm{hole}},
\end{equation}
whose complement has probability at most $3\delta/8$,
\begin{equation}
\label{eq:mutual-knn-upper-implication}
    \Delta>M_{\mathrm{high}}h
    \quad\Longrightarrow\quad
    \mathsf{Cross}_{\mathrm{mut}}^{\,c}.
\end{equation}
\end{proposition}

\begin{proof}
The first task is to express every $k$NN radius in units of $h$. On $\mathcal E_{\mathrm{sep}}$, for every $z\in S$,
\begin{align}
\label{eq:mutual-knn-radius-fill-units}
    \left(\frac{D_k(z)}{h}\right)^d
    &\le
    \frac{r^d}{c_{\mathrm{fill}}^d\ell/m}
    \notag\\
    &\le
    \frac{3AL}{(1-\varepsilon)\underline c\,m}
    \frac{m}{c_{\mathrm{fill}}^d\ell}
    &&\text{since }2k\le3AL
    \notag\\
    &=
    \frac{12AR_n\overline c}{(1-\varepsilon)\underline c}
    =M_{\mathrm{high}}^d.
    &&\text{since }c_{\mathrm{fill}}^d=(4\overline c)^{-1}
\end{align}
Thus every neighborhood radius is at most $M_{\mathrm{high}}h$. Suppose now that $\Delta>M_{\mathrm{high}}h$. For any $z\in S_1$ and $w\in S_2$,
\begin{equation}
\label{eq:mutual-knn-cross-outside-radius}
    \|z-w\|\ge\Delta
    >M_{\mathrm{high}}h
    \ge D_k(z),
\end{equation}
so $w\notin N_k(z)$. Reversing the roles of the two points gives $z\notin N_k(w)$. Thus no directed neighbor relation crosses between the components. Since a mutual edge requires both directed listings, no mutual cross edge can occur.
\end{proof}

\begin{proof}[Proof of Theorem~\ref{thm:mutual-knn-threshold}(i)]
Proposition~\ref{prop:mutual-knn-no-cross} gives
\begin{equation*}
    \{\Delta>M_{\mathrm{high}}h\}
    \cap\mathsf{Cross}_{\mathrm{mut}}
    \subseteq\mathcal E_{\mathrm{sep}}^c.
\end{equation*}
Since $\mathbb P(\mathcal E_{\mathrm{sep}}^c)\le\delta/4+\delta/8$, taking probabilities proves \eqref{eq:mutual-knn-main-upper}.
\end{proof}

\subsection{Existence of a mutual cross-component edge}
\label{app:mutual-knn-cross}

The lower direction requires a different argument. Rather than controlling all cross pairs, we locate one observed pair that is forced to become mutual. The construction has three parts. First, the sample must place one point near each endpoint of a closest manifold pair. Second, the resulting short cross pair must have fewer than $k$ same-component competitors at either endpoint. Third, because exact distance ties are possible, we choose the closest pair so that the two one-sided listings hold simultaneously. Throughout this subsection, write
\begin{equation}
\label{eq:mutual-knn-bridge-radius}
    t:=B\Delta.
\end{equation}

\begin{lemma}[A short observed cross pair]
\label{lem:mutual-knn-cap}
Let $(x_0,y_0)\in\mathcal M_1\times\mathcal M_2$ attain the offset and set
\begin{equation}
\label{eq:mutual-knn-caps}
    U:=\mathcal M_1\cap B(x_0,a\Delta),
    \qquad
    V:=\mathcal M_2\cap B(y_0,a\Delta).
\end{equation}
Let $\mathcal E_{\mathrm{cap}}$ be the event that both caps contain a sample. Then
\begin{equation}
\label{eq:mutual-knn-cap-probability}
    \mathbb P(\mathcal E_{\mathrm{cap}}^c)
    \le2e^{-\underline c a^dm\Delta^d}.
\end{equation}
Moreover, on $\mathcal E_{\mathrm{cap}}$, the closest observed cross distance
\begin{equation}
\label{eq:mutual-knn-cross-distance}
    d_\times:=\min_{x\in S_1,\,y\in S_2}\|x-y\|
\end{equation}
satisfies $d_\times\le t$.
\end{lemma}

\begin{proof}
We first ensure that the sample sees the region where the two manifolds come closest. Since $a\Delta<t\le r_*$, the local lower mass bound gives $\mu_1(U),\mu_2(V)\ge\underline c(a\Delta)^d$. For the first cap,
\begin{align*}
    \mathbb P(S_1\cap U=\varnothing)
    &=(1-\mu_1(U))^{n_1}
    \le e^{-n_1\mu_1(U)}
    \le e^{-\underline c a^dm\Delta^d},
\end{align*}
and the same estimate holds for $V$. A union bound proves \eqref{eq:mutual-knn-cap-probability}. On the cap event, choose $x_{\mathrm{cap}}\in S_1\cap U$ and $y_{\mathrm{cap}}\in S_2\cap V$. Then
\begin{equation*}
    d_\times
    \le\|x_{\mathrm{cap}}-y_{\mathrm{cap}}\|
    \le a\Delta+\Delta+a\Delta
    =B\Delta=t.
\end{equation*}
\end{proof}

\begin{lemma}[Uniform same-component crowding]
\label{lem:mutual-knn-crowding}
Suppose
\begin{equation}
\label{eq:mutual-knn-crowding-condition}
    (n_i-1)\overline c\,t^d\le\frac{k}{4},
    \qquad i\in\{1,2\}.
\end{equation}
Then the event
\begin{equation}
\label{eq:mutual-knn-crowding-event}
    \mathcal E_{\mathrm{crowd}}
    :=\{Q(z;t)\le k-1\text{ for every }z\in S\}
\end{equation}
satisfies
\begin{equation}
\label{eq:mutual-knn-crowding-probability}
    \mathbb P(\mathcal E_{\mathrm{crowd}}^c)
    \le e^{-\gamma k}.
\end{equation}
\end{lemma}

The left-hand side of \eqref{eq:mutual-knn-crowding-condition} bounds the expected number of same-component points within the possible bridge distance $t$. The threshold $k/4$ leaves a constant-factor margin: if the mean is at most $k/4$, observing $k$ competitors requires a fourfold upper-tail deviation and is exponentially unlikely in $k$.

\begin{proof}
Fix a labeled point $z=X_{ij}\in S_i$. Conditional on $z=x$, the remaining same-component observations are independent with law $\mu_i$, and hence
\begin{equation*}
    Q(z;t)\mid\{z=x\}
    \sim
    \operatorname{Bin}\bigl(n_i-1,\mu_i(B(x,t))\bigr).
\end{equation*}
Because $t\le r_*$, the upper mass bound and \eqref{eq:mutual-knn-crowding-condition} give
\begin{equation*}
    \mathbb E[Q(z;t)\mid z=x]
    \le(n_i-1)\overline c\,t^d
    \le\frac{k}{4}.
\end{equation*}
The event $Q(z;t)\ge k$ is therefore at least a fourfold deviation above the worst-case mean. The Chernoff upper tail gives
\begin{align}
\label{eq:mutual-knn-crowding-one-point}
    \mathbb P(Q(z;t)\ge k)
    &\le\left(\frac{e (k/4)}{k}\right)^k = \left(\frac{e}{4}\right)^k.
\end{align}
A union bound over all labels, together with $\log n\le L\le k/A$, gives
\begin{align*}
    \mathbb P(\mathcal E_{\mathrm{crowd}}^c)
    &\le n\left(\frac e4\right)^k
    =\exp\left(\log n-k\log\frac4e\right)
    \notag\\
    &\le
    \exp\left[-\left(\log\frac4e-\frac1A\right)k\right]
    =e^{-\gamma k}.
\end{align*}
Finally, $A\ge3$ implies $\gamma\ge\log(4/e)-1/3>1/20$.
\end{proof}

\begin{proposition}[A closest pair is a mutual edge]
\label{prop:mutual-knn-tie-bridge}
On $\mathcal E_{\mathrm{cap}}\cap\mathcal E_{\mathrm{crowd}}$, the graph $G_{\mathrm{mut}}$ contains a cross edge.
\end{proposition}

\begin{proof}
Let
\begin{equation}
\label{eq:mutual-knn-minimizing-pairs}
    \mathcal P
    :=\{(x,y)\in S_1\times S_2:\|x-y\|=d_\times\}.
\end{equation}
Among all pairs in $\mathcal P$, choose $x^\star$ as the smallest-labeled $S_1$ endpoint appearing in any minimizing pair. Having fixed $x^\star$, choose $y^\star$ as its smallest-labeled minimizing partner in $S_2$. This asymmetric choice is deliberate: the second choice resolves ties in the neighbor list of $x^\star$, while the first resolves ties in the neighbor list of $y^\star$. By Lemma~\ref{lem:mutual-knn-cap},
\begin{equation}
\label{eq:mutual-knn-star-distance}
    \|x^\star-y^\star\|=d_\times\le t.
\end{equation}

We first show that $x^\star$ lists $y^\star$. No point of $S_2$ can precede $y^\star$ in the neighbor order of $x^\star$: a strictly closer point would contradict the minimality of $d_\times$, while an equally distant point with a smaller label would contradict the choice of $y^\star$. Every predecessor therefore belongs to $S_1$ and, by \eqref{eq:mutual-knn-star-distance}, lies within distance $t$ of $x^\star$. Such points are counted by $Q(x^\star;t)$, so
\begin{equation}
\label{eq:mutual-knn-first-listing}
    \#\{w:w\text{ precedes }y^\star\text{ from }x^\star\}
    \le Q(x^\star;t)\le k-1.
\end{equation}
Thus $y^\star\in N_k(x^\star)$.

For the reverse direction, the same argument holds.
Thus,
\begin{equation}
\label{eq:mutual-knn-second-listing}
    \#\{w:w\text{ precedes }x^\star\text{ from }y^\star\}
    \le Q(y^\star;t)\le k-1.
\end{equation}
Hence $x^\star\in N_k(y^\star)$, and $\{x^\star,y^\star\}\in E_{\mathrm{mut}}$ is a cross edge.
\end{proof}

\begin{proposition}[The low ratio implies the crowding condition]
\label{prop:mutual-knn-ratio-crowding}
Assume \eqref{eq:mutual-knn-balance}. On $\mathcal E_{\mathrm{cover}}$,
\begin{equation}
\label{eq:mutual-knn-ratio-crowding-implication}
    \Delta<M_{\mathrm{low}}h
    \quad\Longrightarrow\quad
    \eqref{eq:mutual-knn-crowding-condition}\text{ holds}.
\end{equation}
\end{proposition}

\begin{proof}
It remains to convert the relative condition $\Delta<M_{\mathrm{low}}h$ into an absolute crowding bound at radius $t=B\Delta$. Fix $i\in\{1,2\}$. Raising the ratio condition to the $d$th power and then using $h^d\le C_{\mathrm{fill}}^d\ell/m$ on $\mathcal E_{\mathrm{cover}}$ gives
\begin{align}
\label{eq:mutual-knn-ratio-crowding-calc}
    (n_i-1)\overline c\,t^d
    &<
    (n_i-1)\overline c B^d
    M_{\mathrm{low}}^d h^d
    \notag\\
    &\le
    (n_i-1)\overline c B^d
    \left(\frac{AR_n\underline c}{2^{d+4}\overline c B^d}\right)
    \left(\frac{2^{d+1}}{\underline c}\frac{\ell}{m}\right)
    \notag\\
    &=
    \frac{n_i-1}{m}\frac{AR_n\ell}{8}
    =
    \frac{n_i-1}{m}\frac{AL}{8}
    \notag\\
    &\le2\frac{AL}{8}
    \le\frac{k}{4}.
    &&\text{balance and }AL\le k
\end{align}
The calculation is arranged so that the geometric factors cancel: $\underline c$, $\overline c$, and $B^d$ disappear, while $2^{d+1}/2^{d+4}=1/8$. The remaining sample-balance factor is at most $2$, leaving the desired $k/4$ threshold. This proves \eqref{eq:mutual-knn-crowding-condition} for both components. Notice that the balance assumption enters only through $(n_i-1)/m\le n_{\max}/n_{\min}\le2$; it is not used in the upper graph comparison.
\end{proof}

\begin{proof}[Proof of Theorem~\ref{thm:mutual-knn-threshold}(ii)]
We can now assemble the lower-threshold result. The crowding condition \eqref{eq:mutual-knn-crowding-condition} is deterministic: it depends on $n_1,n_2,k,\Delta$ and the regularity constants, but not on the realized sample. Exactly one of the following two cases therefore holds.

Suppose first that the crowding condition fails. Proposition~\ref{prop:mutual-knn-ratio-crowding} says that, on $\mathcal E_{\mathrm{cover}}$, the low-ratio condition would imply crowding. Its contrapositive says that the low-ratio condition cannot hold on this cover event. Equivalently,
\begin{equation}
\label{eq:mutual-knn-contrapositive}
    \{\Delta<M_{\mathrm{low}}h\}
    \subseteq
    \mathcal E_{\mathrm{cover}}^c.
\end{equation}
Hence the joint event in \eqref{eq:mutual-knn-main-lower} has probability at most $\delta/4$.

Suppose instead that the crowding condition holds. Then the cap and uniform-crowding events place us exactly in the setting of Proposition~\ref{prop:mutual-knn-tie-bridge}, which produces a mutual cross edge. Therefore
\begin{align*}
    \mathbb P(\mathsf{Cross}_{\mathrm{mut}}^c)
    &\le
    \mathbb P(\mathcal E_{\mathrm{cap}}^c)
    +\mathbb P(\mathcal E_{\mathrm{crowd}}^c)
    \notag\\
    &\le
    2e^{-\underline c a^dm\Delta^d}
    +e^{-\gamma k}.
\end{align*}
In the first case the low-ratio event itself costs at most $\delta/4$; in the second, failure of the bridge is controlled by the cap and crowding errors. Adding the $\delta/4$ term produces a bound valid in either case, which is \eqref{eq:mutual-knn-main-lower}.
\end{proof}

\begin{remark}[Asymptotic threshold bracket]
Writing $M:=(\overline c/\underline c)^{1/d}$, we obtain
\begin{equation*}
    M_{\mathrm{high}}
    =\Theta\bigl((AR_n)^{1/d}M\bigr),
    \qquad
    M_{\mathrm{low}}
    =\Theta\left(\frac{(AR_n)^{1/d}}{BM}\right).
\end{equation*}
\end{remark}

\subsection{Consequences and extensions}
\label{app:mutual-knn-extensions}

\begin{lemma}[Accuracy of the observed sample offset]
\label{lem:mutual-knn-sample-offset}
Let
\begin{equation}
\label{eq:mutual-knn-sample-offset}
    \Delta_{\mathrm{sam}}
    :=\min_{x\in S_1,\,y\in S_2}\|x-y\|.
\end{equation}
For every sample realized offset,
\begin{equation}
\label{eq:mutual-knn-sample-offset-bound}
    \Delta
    \le\Delta_{\mathrm{sam}}
    \le\Delta+h_1+h_2
    \le\Delta+2h_{\max}.
\end{equation}
\end{lemma}

\begin{proof}
The lower bound is immediate: every observed cross pair is also a pair in $\mathcal M_1\times\mathcal M_2$, and hence has distance at least $\Delta$. For the upper bound, let $(p_1^*,p_2^*)$ attain the manifold offset. By the definition of $h_i$, choose $x_i\in S_i$ with $\|x_i-p_i^*\|\le h_i$. Then
\begin{align*}
    \Delta_{\mathrm{sam}}
    &\le\|x_1-x_2\|
    \notag\\
    &\le
    \|x_1-p_1^*\|+\|p_1^*-p_2^*\|+\|p_2^*-x_2\|
    \notag\\
    &\le h_1+\Delta+h_2
    \le\Delta+2h_{\max}.
\end{align*}
\end{proof}

\begin{corollary}[Bounded perturbations preserve the no-edge direction]
\label{cor:mutual-knn-noise}
Suppose the observed points are $\widetilde z=z+\xi_z$ with $\|\xi_z\|\le\sigma$ and the graph is built from the perturbed observations. On $\mathcal E_{\mathrm{sep}}$, a sufficient condition for no perturbed mutual cross edge is
\begin{equation}
\label{eq:mutual-knn-noise-condition}
    \Delta>M_{\mathrm{high}}h_{\max}+4\sigma.
\end{equation}
\end{corollary}

\begin{proof}
For any clean pair $z,w$,
\begin{equation*}
    \bigl|\|\widetilde z-\widetilde w\|-\|z-w\|\bigr|
    \le\|\xi_z\|+\|\xi_w\|
    \le2\sigma.
\end{equation*}
Hence perturbed cross distances are at least $\Delta-2\sigma$. Moreover, the $k$ clean neighbors within $D_k(z)$ all lie within $D_k(z)+2\sigma$ of $\widetilde z$, so $\widetilde D_k(\widetilde z)\le D_k(z)+2\sigma$. On $\mathcal E_{\mathrm{sep}}$, condition \eqref{eq:mutual-knn-noise-condition} gives
\begin{equation*}
    \Delta-2\sigma
    >M_{\mathrm{high}}h_{\max}+2\sigma
    \ge\widetilde D_k(\widetilde z),
\end{equation*}
and Proposition~\ref{prop:mutual-knn-no-cross} applies.
\end{proof}

\section{Implementation details}
\label{app:impl}

This appendix specifies the preprocessing, local-degree and sweep-endpoint rules, filtration summaries, label construction, and baseline grids omitted from the main text.

\paragraph{Preprocessing and effective dimension.}
When specified, we standardize each coordinate and project onto the smallest PCA subspace explaining $90\%$ of the variance, capped at $64$ dimensions. We estimate $d_{\mathrm{eff}}$ by the same $90\%$-variance rule on the resulting representation, or directly on the supplied representation when preprocessing is disabled.

\paragraph{Per-edge gate.}
For a reciprocal edge, $j\in N_{k_i}(i)$ and $i\in N_{k_j}(j)$ imply
$\|x_i-x_j\|\le\min(H_i,H_j)$, so the $\alpha$-gate is automatic when $\alpha\ge1$. The gate therefore only acts on the non-reciprocal fallback edges added when an isolated node is connected to its nearest neighbor; in that case it filters edges whose endpoints sit on opposite sides of a density discontinuity (a sparse-region point pulling in a dense-region neighbor). We use $\alpha=1.5$ in the pilot graph and we disable both fallback and gate during the persistence sweep.

\paragraph{Local-degree schedule and pilot graph.}
The implementation uses
$$
H_{\mathrm{ref}}
=
\operatorname{median}\{H_i^{\mathrm{pilot}}:H_i^{\mathrm{pilot}}>0\},
\qquad
k_{\max}
=
\min\{n-1,\lfloor4k^\star\rfloor\},
$$
and clips Eq.~\eqref{eq:local-k} to
$$
k_i^\star
=
\operatorname{clip}_{[k^\star,k_{\max}]}
\left(
\left\lfloor
k^\star
\left(
\frac{H_{\mathrm{ref}}}{H_i^{\mathrm{pilot}}}
\right)^{d_{\mathrm{eff}}}
\right\rfloor
\right).
$$
Neighbor lists and radii $H_i$ are then recomputed on the retained observations. The resulting gated pilot graph supplies the partition used for $\widehat\rho$, not the graphs used for the final bracket.

\paragraph{Offset-to-radius proxy and sweep endpoints.}
Let
$
\widehat{\mathcal C}_1,\ldots,\widehat{\mathcal C}_{K_0}
$
be the components of the pilot graph and define
$$
\widehat h
=
\operatorname{Quantile}_{q_h}
\bigl\{
H_i:i\in\mathcal A,\ H_i>0
\bigr\},
\qquad
q_h=0.5
$$
by default. If $K_0>1$, the numerator $\widehat\Delta$ is the smallest
observed distance between two pilot components and
$\widehat\rho=\widehat\Delta/\widehat h$. If $K_0=1$, the
between-component distance is undefined; the implementation records
$\widehat\rho$ as unavailable and uses the connected-pilot branch below.

For the practical endpoint map, we retain the
$A^{1/d_{\mathrm{eff}}}$ dependence of the thresholds in
Rem.~\ref{rem:constants-threshold} and absorb the unknown geometric factors
into the theory-informed curves
$$
\overline C(A,d_{\mathrm{eff}})
=
\frac{2}{(1-\varepsilon)^{1/d_{\mathrm{eff}}}}
(2A)^{1/d_{\mathrm{eff}}},
\qquad
\underline C(A,d_{\mathrm{eff}})
=
\left(
\frac{A}
{2^{d_{\mathrm{eff}}+2}B^{d_{\mathrm{eff}}}}
\right)^{1/d_{\mathrm{eff}}},
\qquad
B=1+2a.
$$
The curve $\overline C$ gives the practical no-bridge boundary and
$\underline C$ the practical bridge boundary. At the default anchor
$A_0=1$, the un-clipped endpoints are
$$
(\widetilde A_{\mathrm{low}},\widetilde A_{\mathrm{high}})
=
\begin{cases}
(0.85A_0,\,1.15A_0),
&
\widehat\rho
\ge
\overline C(A_0,d_{\mathrm{eff}}),
\\
(0.15A_0,\,1.10A_0),
&
\widehat\rho\text{ undefined or }
\widehat\rho
\le
\underline C(A_0,d_{\mathrm{eff}}),
\\
(A_-(\widehat\rho),\,A_+(\widehat\rho)),
&
\underline C(A_0,d_{\mathrm{eff}})
<
\widehat\rho
<
\overline C(A_0,d_{\mathrm{eff}}).
\end{cases}
$$
Here the transitional endpoints are obtained by inverting the two practical
threshold curves:
$$
\overline C\bigl(A_-(r),d_{\mathrm{eff}}\bigr)=r,
\qquad
\underline C\bigl(A_+(r),d_{\mathrm{eff}}\bigr)=r.
$$
The resulting interval is clipped to $[0.15A_0,4A_0]$ and enlarged if
necessary to contain $A_0$.

\paragraph{Minimum-component size.}
The size floor
\begin{equation}
\label{eq:smin}
s_{\min}
=
\max\bigl(\lceil0.005|\mathcal A|\rceil,\ k^\star,\ 5\bigr)
\end{equation}
is applied to $K_{\mathrm{big}}(k)$ and $\widehat K$, while $K_{\mathrm{raw}}(k)$ remains unfiltered. Its three terms provide a relative floor, the logarithmic sample support suggested by the theory, and a small-sample safeguard, respectively.

\paragraph{Mass-bounded alternative bracket.}
The mass-bounded count $K_{\mathrm{mass},\gamma}(k)$ is the smallest number of largest graph components needed to contain at least a fraction $\gamma$ of the retained observations. We compute it at the same sampled scales as $K_{\mathrm{big}}(k)$, using $\gamma=0.95$ by default. The ordinary mass bracket is its minimum and maximum over these scales.

An optional run-length variant forms the bracket using only count values that occur at two or more consecutive sampled scales. If no value meets this condition, it returns the ordinary mass bracket. This filtering can narrow the interval, but can also remove the reference count. We therefore report the ordinary mass bracket by default and retain the run-length variant as an additional diagnostic. App.~\ref{app:mass-bracket} reports the mass-bracket sensitivity analysis.

\paragraph{Conservative point estimate $\widehat K$.}
A single point estimate of the number of clusters present can be read off the same filtration as the bracket via a standard $0$-th persistence perspective,
\begin{equation}
\label{eq:K-persistent}
\widehat K \;:=\; \#\bigl\{c : k_{\mathrm{birth}}(c) \le k_{\mathrm{low}},\ k_{\mathrm{death}}(c) > k_{\mathrm{high}}\bigr\},
\end{equation}
restricted to components of size $\ge s_{\min}$ --- in other words, we consider those clusters that persist across the entire uncertainty zone.

\paragraph{Practical estimate $\widehat K_{\mathrm{prac}}$.}
The strict $\widehat K$ from Eq.~\eqref{eq:K-persistent} is $1$ in the non-separable regime, which is the conservative answer the threshold theorem admits. When the bracket spans $\{1, 2, \ldots, K_{\mathrm{high}}\}$ with $K_{\mathrm{high}} > 1$, however, downstream applications usually need a labeled partition, so reporting $K = 1$ is not helpful. 
Define the observed non-trivial count set
\[
\mathcal V
=\{K_{\mathrm{big}}(k):k\in\mathcal K,\
  K_{\mathrm{big}}(k)>1\}.
\]
The practical estimate is
\begin{equation}
\label{eq:k-prac}
\widehat K_{\mathrm{prac}}=
\begin{cases}
\widehat K, & \widehat K\ge2,\\
\displaystyle
\max\operatorname*{arg\,max}_{K\in\mathcal V}
\#\{k\in\mathcal K:K_{\mathrm{big}}(k)=K\},
  & \widehat K<2,\ \mathcal V\ne\varnothing,\\
1, & \text{otherwise}.
\end{cases}
\end{equation}

When $\widehat K<2$, this selects the most frequent non-trivial value of $K_{\mathrm{big}}$ on the discrete sweep grid $\mathcal K$, breaking ties toward larger $K$, since under-merging is recoverable from the same labels while over-merging is not. If no such value occurs, $\widehat K_{\mathrm{prac}}=1$. 

\paragraph{Representative labels.}
To construct the practical partition for $K=\widehat K_{\mathrm{prac}}$, we first seek a sampled scale with $K_{\mathrm{big}}(k)=K$. Among qualifying scales, we choose the one nearest the integer part of the median sampled rank. If none qualifies, we choose a scale with at least $K$ raw components, first minimizing the excess component count and then the distance to the median rank. Ties favor the earlier sampled rank.

For $K>1$, we retain the $K$ largest active components and label the remaining active observations as noise ($-1$). For $K=1$, we combine the active observations into one group. Observations originally removed by density pruning may subsequently receive labels through the radius-limited neighbor vote; active observations demoted to noise are not reassigned by this step.

Sweep graphs omit the pilot graph's isolated-node fallback edges and per-edge gate. These pilot operations can nevertheless affect the selected sweep range through $\widehat\rho$.

\paragraph{Baseline parameter grids.}
The DBSCAN grid sweeps the neighborhood radius over the multiples $\{0.7,\ 1.0,\ 1.5\}$ of the median $k$-distance, crossed with two minimum-sample settings ($5$ and $5\log n$), for six configurations. The HDBSCAN grid sweeps the cluster-selection method over $\{\textit{eom},\ \textit{leaf}\}$ crossed with the minimum-cluster-size at $\{0.5,\ 1,\ 2,\ 5\}\%$ of $n$, for eight configurations. Each grid bracket is $[\min_g K_g, \max_g K_g]$ across configurations and is compared to MBC on equal footing in \S\ref{sec:exp-calibration}. The single-$K$ baselines KMeans, GMM, Ward, and Spectral are run at $K = K^\star$ when ground truth is known; Spectral is skipped for $n > 3{,}000$ for memory reasons.

\section{Full pseudocode}
\label{app:full-algorithm}

The compact pseudocode in the main paper omits the local-$k$ schedule, isolated-node fallback, and the bracket-endpoint construction. The full 
procedure is given below.

\begin{algorithm}[h]
\caption{Bracket estimator (full)}
\label{alg:mbc-full}
\begin{algorithmic}[1]
\Require $X\in\mathbb R^{n\times D}$, $\delta\in(0,1)$, $q$, $\alpha_q$, $\alpha$
\State \textbf{Preprocess.} Apply the benchmark-specific standardization and PCA rule;\ estimate $d_{\mathrm{eff}}$ on the resulting representation
\State \textbf{Pilot quantities.} Set $k^\star\gets\lceil\log(4n/\delta)\rceil$;\ compute the full-sample pilot radii $H_i^{\mathrm{pilot}}\gets D_{k^\star}(x_i)$
\State \textbf{Pruning.} Compute the retained set $\mathcal A$ from Eq.~\eqref{eq:density-prune}, including the small-sample pruning safeguard
\State \textbf{Local degrees.} Compute the clipped local-degree template $k_i^\star$ from Eq.~\eqref{eq:local-k}
\State \textbf{Pilot graph.} Recompute neighbors on $\mathcal A$ at degrees $k_i^\star$;\ form reciprocal edges, add one nearest-neighbor fallback for each isolated observation, and apply the $\alpha$-gate;\ record the pilot components and graph radii $H_i$
\If{the pilot graph has more than one component}
    \State Estimate the smallest between-component distance and set $\widehat\rho$ equal to this distance divided by the selected quantile of $\{H_i:i\in\mathcal A\}$
\Else
    \State Mark $\widehat\rho$ undefined and use the connected-graph endpoint branch
\EndIf
\State \textbf{Persistence sweep range.} Apply the practical threshold map and coefficient clipping from App.~\ref{app:impl};\ obtain $k_{\mathrm{low}},k_{\mathrm{high}}$ from Eq.~\eqref{eq:bracket-endpoints}
\State \textbf{Persistence sweep grid.} Choose at most $15$ approximately equally spaced integer scales $\mathcal K\subseteq[k_{\mathrm{low}},k_{\mathrm{high}}]$
\For{$k\in\mathcal K$}
    \State Scale the local degrees by $k/k^\star$ and apply the sweep clipping
    \State Build the mutual-neighbor graph on $\mathcal A$ without fallback or gate
    \State Record raw, size-filtered, and mass-bounded component counts and component lifetimes
\EndFor
\State \textbf{Summaries.} Compute the primary, raw, ordinary mass, and run-length mass brackets;\ compute $\widehat K$ and $\widehat K_{\mathrm{prac}}$
\State \textbf{Labels.} Select a representative sweep scale for $\widehat K_{\mathrm{prac}}$;\ retain the required largest components and reassign eligible density-pruned observations according to App.~\ref{app:impl}
\State \Return brackets, point estimates, representative labels, and diagnostics
\end{algorithmic}
\end{algorithm}

\section{Mechanism ablations: $A$-sweep, reciprocity, and pruning}
\label{app:A-sweep}

This appendix isolates the effects of the logarithmic coefficient $A$, mutual reciprocity, and density pruning while holding all other settings fixed. We use three seeds on five $n=2000$ datasets spanning clean low- and high-dimensional separation, background contamination, non-separable moons, and hierarchical structure.

\subsection{$A$-sweep on representative datasets}

Table~\ref{tab:A-sweep} reports the median bracket, $\widehat K_{\mathrm{prac}}$, $\widehat\rho$, regime flag, and ARI as $A$ varies over $\{0.5, 1, 1.5, 2, 4, 8, 12\}$. The proof-level coefficient at $\varepsilon = 1/2$ is $A_{\mathrm{cert}} = 12$, included as the right endpoint of the sweep.

\begin{table}[h]
\centering
\small
\caption{$A$-sweep on five representative datasets. Entries are median primary brackets over three seeds; $A=1$ is the practical default and $A=12$ the proof-level coefficient at $\varepsilon=1/2$.}
\label{tab:A-sweep}
\begin{tabular}{lcccc}
\toprule
dataset & $A=0.5$ & $A=1$ & $A=4$ & $A=12$ \\
\midrule
\textit{blobs-2D-clean} & $[4,5]$ & $[4,4]$ & $[3,3]$ & $[1,4]$ \\
\textit{blobs-2D-bg10} & $[1,8]$ & $[1,1]$ & $[1,2]$ & $[1,1]$ \\
\textit{moons-n0.15} & $[1,1]$ & $[1,3]$ & $[1,1]$ & $[1,1]$ \\
\textit{blobs-50D-easy} & $[6,6]$ & $[6,6]$ & $[1,6]$ & $[1,6]$ \\
\bottomrule
\end{tabular}
\end{table}

Three patterns are visible. First, the empirical transition between fragmenting ($K_{\mathrm{high}} > K^\star$) and over-connecting ($K_{\mathrm{low}} = 1$) sits near $A = 1$ on the clean separable cases: $A = 1$ is the only row that pins the bracket at $[K^\star, K^\star]$ in the separable regime on \textit{blobs-2D-clean}. Second, increasing $A$ monotonically pushes the bracket toward $[1, K_{\mathrm{high}}]$ on every dataset with well-separated components—by $A_{\mathrm{cert}} = 12$, every non-noisy dataset has $K_{\mathrm{low}} = 1$, the over-connection mode the proof-level coefficient is expected to enter. Third, the regime flag and $\widehat\rho$ track $A$ in the predicted direction: $\widehat\rho$ peaks near the empirical transition and decays as $A$ grows past it, and the flag flips from transitional (at $A = 0.5$) to separable (near $A = 1$) and back to transitional or non-separable as $k^\star$ overshoots. The contamination case (\textit{blobs-2D-bg10}) is the exception: pruning removes only $\approx 1\%$ of points under the background noise model, so the bracket collapses at $A = 1$ and never recovers—consistent with the limits discussed below.

\subsection{Mutual versus union $k$NN at the primary scale}

Table~\ref{tab:mutual-union} compares the mutual graph used by the primary against its union counterpart at the same scale ($A = 1$, all other parameters fixed). On the separable cases the two graphs return identical brackets—the bridge events that reciprocity is designed to suppress simply do not occur in the separable regime. The differentiator is the non-separable two-moons case: the mutual graph reports $[1, 3]$—a wide bracket that exposes the geometric ambiguity—while the union graph collapses to $[1, 1]$, having fused the two moons through their bridge points. This is the predicted role of reciprocity: it does not change the answer when components are well-separated, and it preserves bracket informativeness when they are not.

\begin{table}[h]
\centering
\small
\caption{Mutual versus union $k$NN at $A=1$, shown for a separated control and the bridge-sensitive moons case.}
\label{tab:mutual-union}
\begin{tabular}{lcc}
\toprule
dataset & mutual bracket & union bracket \\
\midrule
\textit{blobs-2D-clean} & $[4,4]$ & $[4,4]$ \\
\textit{moons-n0.15} & $[1,3]$ & $[1,1]$ \\
\bottomrule
\end{tabular}
\end{table}

\subsection{Density pruning on the contamination case}

Table~\ref{tab:pruning-ablation} contrasts the bracket on \textit{blobs-2D-bg10} and \textit{moons-n0.15} with and without Eq.~\eqref{eq:density-prune}. The two cases stress different failure modes: the first violates the manifold-supported sampling assumption through background noise; the second simply has overlapping components that no contamination model can repair.

\begin{table}[h]
\centering
\small
\caption{Density-pruning ablation at $A=1$. Brackets and $\widehat K_{\mathrm{prac}}$ are medians over seeds; ARI is averaged over seeds.}
\label{tab:pruning-ablation}
\begin{tabular}{llccc}
\toprule
dataset & pruning & median bracket & median $\widehat K_{\mathrm{prac}}$ & mean ARI \\
\midrule
\textit{blobs-2D-bg10} & on  & $[1,1]$ & $1$ & $0.11$ \\
\textit{blobs-2D-bg10} & off & $[1,3]$ & $3$ & $0.51$ \\
\textit{moons-n0.15}   & on  & $[1,3]$ & $2$ & $0.17$ \\
\textit{moons-n0.15}   & off & $[1,2]$ & $2$ & $0.17$ \\
\bottomrule
\end{tabular}
\end{table}

The pruning rule is most useful when off-manifold points have noticeably larger pilot radii than the sampled manifold. The uniform background noise used here barely creates that separation: most contaminating points sit close to one of the four blob components, so the empirical $q$-quantile of the pilot $k$NN radii separates little of the background from the data, and only $\approx 1\%$ of points are pruned. Under such mild pruning, the active-set restriction can also remove boundary points that were helping keep component structure visible in the filtration. The on/off comparison is therefore mixed across seeds. This ablation should be read as a stress test of pruning's range of applicability: pruning is a useful stabilizer when large-radius contaminants are identifiable, but it is not a general remedy for overlapping or near-manifold noise.

\section{Sensitivity to $\delta$ and $\alpha$}
\label{app:sensitivity}

Table~\ref{tab:sensitivity} evaluates three diagnostic datasets over $\delta\in\{0.01,0.05,0.10,0.20\}$ and $\alpha\in\{1.0,1.25,1.5,2.0\}$.

\begin{table}[h]
\centering
\small
\caption{Sensitivity grid across $4\times 4$ values of $(\delta, \alpha)$ and three seeds. Bracket cells report endpoint-wise min / median / max across the grid, not three literal brackets.}
\label{tab:sensitivity}
\resizebox{\textwidth}{!}{%
\begin{tabular}{lcccc}
\toprule
dataset & regime & bracket (min / med / max) & $\widehat K_{\mathrm{prac}}$ (min / med / max) & ARI (min / med / max) \\
\midrule
50D blobs (easy)               & stable & $[6, 6]$ / $[6, 6]$ / $[6, 6]$       & 5 / 6 / 6     & 0.996 / 1.000 / 1.000 \\
Retina (labeled)               & sensitive   & $[1, 6]$ / $[1, 8]$ / $[1, 9]$     & 2 / 6.5 / 9   & 0.013 / 0.542 / 0.639 \\
two moons ($\sigma_{\mathrm{noise}} = 0.10$) & sensitive & $[1, 1]$ / $[1, 3.5]$ / $[2, 5]$ & 1 / 2 / 5 & 0.000 / 0.972 / 0.996 \\
\bottomrule
\end{tabular}}
\end{table}

The easy $50$D blobs remain fixed at $[6,6]$, while retina and noisy moons vary across the grid, consistent with their location near the transitional regime.

\section{Multi-seed stability}
\label{app:multi-seed}

Across seeds $7$, $11$, and $23$, the mean within-dataset standard deviation of ARI was $0.098$ on the $38$ synthetic datasets and $0.005$ on the $11$ real-world datasets; the corresponding standard deviations of $K_{\mathrm{high}}$ were $1.86$ and $0.27$. Variability was mainly concentrated in hierarchical and contaminated datasets.

\section{Perturbation walk}
\label{app:perturbation-walk}

Table~\ref{tab:perturbation-walk} reports the bracket and regime flag along three controlled perturbation axes referenced from \S\ref{sec:exp-perturbations}. The base configuration is a clean four-blob mixture in $\mathbb{R}^2$ with $n = 1500$. Each block holds the other two axes fixed.

\begin{table}[h]
\centering
\small
\caption{Bracket response along three controlled perturbation axes. Entries are medians over seeds for the bracket and $\widehat K_{\mathrm{prac}}$, the modal regime flag, and the mean ARI. The regime flag is a proxy based on $\widehat\rho$; the bracket and labels provide the diagnostic.}
\label{tab:perturbation-walk}
\begin{tabular}{llcccc}
\toprule
axis & setting & bracket & $\widehat K_{\mathrm{prac}}$ & regime & ARI \\
\midrule
\multirow{4}{*}{contamination ($\eta$)}
& $\eta = 0\%$  & $[4, 4]$ & 4 & separable      & 1.00 \\
& $\eta = 5\%$  & $[4, 5]$ & 4 & transitional   & 1.00 \\
& $\eta = 10\%$ & $[1, 3]$ & 2 & separable      & 0.33 \\
& $\eta = 20\%$ & $[1, 4]$ & 4 & non-separable  & 0.33 \\
\midrule
\multirow{4}{*}{cluster spread ($\sigma$ multiplier)}
& $1\times$ & $[4, 4]$ & 4 & separable    & 1.00 \\
& $2\times$ & $[1, 3]$ & 3 & non-separable & 0.31 \\
& $4\times$ & $[1, 1]$ & 1 & non-separable & 0.00 \\
& $8\times$ & $[1, 2]$ & 2 & non-separable & 0.00 \\
\midrule
\multirow{4}{*}{centroid distance ($\Delta$)}
& $\Delta = 10$ & $[4, 4]$ & 4 & separable    & 1.00 \\
& $\Delta = 5$  & $[1, 3]$ & 3 & non-separable & 0.31 \\
& $\Delta = 4$  & $[1, 3]$ & 3 & non-separable & 0.00 \\
& $\Delta = 2$  & $[1, 1]$ & 1 & non-separable & 0.00 \\
\bottomrule
\end{tabular}
\end{table}

Increasing spread or decreasing centroid distance moves the bracket toward smaller counts. Contamination is less monotone because uniform background noise often lies near the blobs rather than forming clearly separated large-radius outliers.

\section{Sampling sweep}
\label{app:sampling-sweep}

Table~\ref{tab:sampling-sweep} reports MBC outputs at $n \in \{200, 500, 1000, 2000, 5000\}$ on three datasets, holding all other parameters fixed. The three datasets span the regimes the theory addresses: a clean 2D separable case, a curved-manifold non-separable case, and a high-dimensional separable case.

\begin{table}[h]
\centering
\small
\caption{Bracket as a function of sample size. Entries are medians over seeds for $\widehat K_{\mathrm{prac}}$, the bracket, $k^\star$, and $\widehat\rho$; the modal regime flag; and the mean ARI.}
\label{tab:sampling-sweep}
\begin{tabular}{lcccccccc}
\toprule
dataset & $n$ & $K^\star$ & $\widehat K_{\mathrm{prac}}$ & bracket & $k^\star$ & $\widehat\rho$ & regime & ARI \\
\midrule
\multirow{5}{*}{2D blobs ($\sigma = 0.9$)}
& 200  & 4 & 3 & $[3, 4]$ & 10 & 3.30  & transitional   & 0.71 \\
& 500  & 4 & 3 & $[3, 4]$ & 11 & 5.48  & separable      & 0.80 \\
& 1000 & 4 & 4 & $[4, 4]$ & 12 & 8.35  & separable      & 0.90 \\
& 2000 & 4 & 4 & $[4, 4]$ & 12 & 8.73  & separable      & 0.90 \\
& 5000 & 4 & 4 & $[4, 4]$ & 13 & 13.50 & separable      & 0.89 \\
\midrule
\multirow{5}{*}{two moons ($\sigma_{\mathrm{noise}} = 0.10$)}
& 200  & 2 & 2 & $[1, 6]$ & 10 & --     & non-separable & 0.60 \\
& 500  & 2 & 2 & $[1, 4]$ & 11 & 1.84   & transitional  & 0.79 \\
& 1000 & 2 & 3 & $[1, 4]$ & 12 & --     & non-separable & 0.61 \\
& 2000 & 2 & 2 & $[1, 4]$ & 12 & --     & non-separable & 0.91 \\
& 5000 & 2 & 2 & $[2, 2]$ & 13 & 6.66   & \textbf{separable} & 0.99 \\
\midrule
\multirow{5}{*}{50D blobs (easy)}
& 200  & 6 & 6 & $[6, 6]$ & 10 & 16.05 & separable    & 0.94 \\
& 500  & 6 & 6 & $[6, 6]$ & 11 & 17.48 & separable    & 0.82 \\
& 1000 & 6 & 5 & $[6, 6]$ & 12 & 1.25  & transitional & 1.00 \\
& 2000 & 6 & 6 & $[6, 6]$ & 12 & 1.15  & transitional & 1.00 \\
& 5000 & 6 & 6 & $[6, 6]$ & 13 & 1.32  & transitional & 0.94 \\
\bottomrule
\end{tabular}
\end{table}

The moons bracket narrows from a wide low-$n$ range to $[2,2]$ at $n=5{,}000$. The high-dimensional primary bracket is already tight, although the practical label count briefly differs at intermediate $n$.

\section{Additional comparisons}
\label{app:additional-results}

\subsection{Alternative interval constructions}

\paragraph{Model-based intervals.}
We compared MBC, the HDBSCAN parameter sweep, and a variational Dirichlet-process Gaussian mixture on $16$ datasets. For the Bayesian comparison, we use the variational approximation of \citet{blei2006variational}, fitting the model on ten $80\%$ subsamples and taking the $5$th--$95$th percentiles of the occupied-component count. The subsampling and percentile construction is our procedure for obtaining an interval; it is not part of the original variational inference method. This comparison tests a model-based source of uncertainty over the cluster count against the geometric MBC bracket.

\begin{table}[h]
\centering
\small
\caption{Intervals induced by the geometric MBC band, an HDBSCAN parameter sweep, and a variational Dirichlet-process mixture on $16$ datasets.}
\label{tab:model-intervals}
\begin{tabular}{llccc}
\toprule
Method & Interval source & Coverage & Median width & Mean width \\
\midrule
MBC
& geometric graph band
& $12/16$ & $0.0$ & $1.50$ \\
HDBSCAN
& parameter sweep
& $15/16$ & $1.5$ & $5.56$ \\
VB DP-GMM
& subsampling and mixture fit
& $11/16$ & $1.0$ & $1.44$ \\
\bottomrule
\end{tabular}
\end{table}

HDBSCAN attains the highest coverage through a substantially wider interval. The variational mixture is slightly narrower on average than MBC but has lower coverage and nonzero median width. Interval-valued output is therefore not unique to MBC; the distinction is that MBC derives its endpoints from the selected graph scales rather than a parameter sweep or subsampling variability.

\subsection{Geometric robustness}

\paragraph{Curvature.}
Positive reach determines the scale below which Euclidean neighborhoods remain locally regular; it does not require globally flat or only mildly curved components. We tested a four-component family of separated quadratic patches at center curvatures
$\{-2,-0.5,0,1,2\}$. The corresponding bracket widths were
$\{0.3,1.3,1.3,1.0,0.3\}$, and the mean ARIs were
$\{0.991,0.991,0.992,0.990,0.990\}$. Every bracket contained the true count. The absence of monotone degradation with curvature is consistent with the theorem: curvature affects the usable local scale and the small-ball constants, while separation is governed by the gap relative to sampling resolution.

\paragraph{Distance metric.}
The main analysis uses Euclidean distance because Theorem~\ref{thm:knn-threshold} is stated in terms of Euclidean balls. We also reran three diagnostic datasets with cosine and Manhattan distance while holding the remaining settings fixed.

\begin{table}[h]
\centering
\small
\caption{Primary MBC bracket under alternative distance metrics. All displayed brackets contain the recorded target count, but their widths differ.}
\label{tab:metric-sensitivity}
\begin{tabular}{lcccc}
\toprule
Dataset & $K^\star$ & Euclidean & Cosine & Manhattan \\
\midrule
2D blobs   & $4$ & $[4,4]$ & $[4,5]$  & $[4,4]$ \\
Two moons  & $2$ & $[2,2]$ & $[1,38]$ & $[2,2]$ \\
100D blobs & $6$ & $[6,6]$ & $[6,6]$  & $[1,6]$ \\
\bottomrule
\end{tabular}
\end{table}

Euclidean distance is best or tied for the narrowest bracket on every row. Cosine distance strongly fragments the two-moons representation, while Manhattan distance widens the high-dimensional blob bracket. Extending the theory to these metrics would require replacing the Euclidean local-mass condition by the corresponding metric-ball bounds.

\subsection{Large-scale image benchmarks and downstream use}

\paragraph{Large-scale image benchmarks.}
On ImageNet-100 ($n\approx5{,}000$), MBC returns $[16,60]$, compared with $[2,111]$ for HDBSCAN and $[38,82]$ for DBSCAN; only the broad HDBSCAN interval contains all $100$ recorded classes in the tested representation. On ImageNet-Dogs ($n=19{,}500$, $K^\star=15$), the corresponding intervals are $[14,19]$, $[2,23]$, and $[19,105]$.

\paragraph{Downstream use.}
We used each ImageNet-Dogs interval as a candidate set for KMeans and spectral clustering. Averaging over every integer in each interval, the MBC candidate set gives the highest mean ARI for both methods. When five approximately equally spaced values are evaluated and the highest-silhouette partition is selected, the MBC interval also gives the highest selected ARI for both methods. The only exception is using silhouette selection for KMeans, where a DBSCAN candidate slightly exceeds the selected MBC candidate despite DBSCAN's lower mean performance.

\section{Synthetic suite: catalog and per-dataset results}
\label{app:synth-catalog}

The synthetic suite contains $38$ datasets in eight families, generated at $n\approx2{,}000$ with $K^\star$ fixed by the generator. Table~\ref{tab:synth-catalog} reports the seed-$7$ results, App.~\ref{app:multi-seed} reports variability, and Table~\ref{tab:per-family} gives the aggregate comparison.

\paragraph{Family rationale.}
\begin{itemize}\itemsep1pt
\item \emph{Classic shapes.} Clean low-dimensional blobs, circles, and moons testing collapse to $[K^\star,K^\star]$ under clear separation.
\item \emph{Additive noise.} Moons and circles at increasing Gaussian noise, testing the transition as within-component dispersion grows.
\item \emph{Background contamination.} Blobs and moons with $5$--$20\%$ uniform background, testing the pruning step under violations of manifold-supported sampling.
\item \emph{Varied scale.} Anisotropic and variance-mismatched blobs and two spirals, testing nonuniform scale and curved geometry.
\item \emph{High-dimensional blobs.} Six-component blobs in $D\in\{50,100,200\}$ and a Swiss roll in $20$D, separating ambient dimension from intrinsic geometry.
\item \emph{Hierarchical.} Nested $3\times3$ blob grids with multiple natural resolutions.
\item \emph{Class imbalance.} Two- to four-cluster mixtures with unequal masses, including small minorities and background contamination.
\item \emph{Adversarial.} Touching, uneven, and mixed-geometry components near the boundary of the separation hypothesis.
\end{itemize}

\begin{table}[h]
\centering
\footnotesize
\setlength{\tabcolsep}{4pt}
\caption{Per-dataset MBC results on the synthetic suite at seed $7$. Primary is the size-filtered bracket and $\gamma{=}0.95$ the mass companion; ``--'' denotes an undefined $\widehat\rho$ because the pilot graph is connected. \emph{S}, \emph{T}, and \emph{N} denote separable, transitional, and non-separable proxy regimes. Bold brackets contain $K^\star$.}
\label{tab:synth-catalog}
\begin{tabular}{llcccccccc}
\toprule
family & dataset & $D$ & $K^\star$ & primary & $\gamma{=}0.95$ & $\widehat K_{\mathrm{prac}}$ & ARI & $\widehat\rho$ & reg. \\
\midrule
\multirow{3}{*}{classic}
& blobs ($\sigma{=}0.9$)             &  2 & 4 & $\boldsymbol{[4, 4]}$  & $\boldsymbol{[4, 4]}$  & 4 & 0.98 & 8.7   & S \\
& circles ($\sigma{=}0.04$)          &  2 & 2 & $\boldsymbol{[2, 2]}$  & $\boldsymbol{[2, 2]}$  & 2 & 1.00 & 5.4   & S \\
& moons ($\sigma{=}0.02$)            &  2 & 2 & $\boldsymbol{[2, 10]}$ & $\boldsymbol{[2, 10]}$ & 2 & 1.00 & 1.6   & T \\
\midrule
\multirow{5}{*}{noise sweep}
& moons ($\sigma{=}0.05$)            &  2 & 2 & $\boldsymbol{[2, 2]}$  & $\boldsymbol{[2, 2]}$  & 2 & 1.00 & 8.5   & S \\
& moons ($\sigma{=}0.10$)            &  2 & 2 & $\boldsymbol{[1, 4]}$  & $\boldsymbol{[1, 4]}$  & 2 & 0.97 & --    & N \\
& moons ($\sigma{=}0.15$)            &  2 & 2 & $\boldsymbol{[1, 3]}$  & $\boldsymbol{[1, 2]}$  & 2 & 0.00 & --    & N \\
& circles ($\sigma{=}0.08$)          &  2 & 2 & $\boldsymbol{[1, 5]}$  & $\boldsymbol{[1, 3]}$  & 5 & 0.64 & 2.0   & T \\
& circles ($\sigma{=}0.15$)          &  2 & 2 & $[1, 1]$  & $[1, 1]$  & 1 & 0.00 & --    & N \\
\midrule
\multirow{6}{*}{bg.\ contamination}
& blobs (bg $5\%$)                   &  2 & 4 & $\boldsymbol{[4, 4]}$  & $\boldsymbol{[4, 4]}$  & 4 & 0.98 & 3.2   & T \\
& blobs (bg $10\%$)                  &  2 & 4 & $[1, 1]$  & $[1, 1]$  & 1 & 0.00 & 4.8   & S \\
& blobs (bg $20\%$)                  &  2 & 4 & $\boldsymbol{[1, 11]}$ & $\boldsymbol{[1, 9]}$  & 11 & 0.71 & --   & N \\
& moons (bg $5\%$)                   &  2 & 2 & $\boldsymbol{[2, 3]}$  & $\boldsymbol{[2, 3]}$  & 2 & 0.96 & 2.7   & T \\
& moons (bg $10\%$)                  &  2 & 2 & $\boldsymbol{[1, 8]}$  & $\boldsymbol{[1, 20]}$ & 2 & 0.88 & 2.0   & T \\
& moons (bg $20\%$)                  &  2 & 2 & $[1, 1]$  & $[1, 1]$  & 1 & 0.01 & 3.1   & T \\
\midrule
\multirow{3}{*}{varied scale}
& blobs (anisotropic)                &  2 & 4 & $\boldsymbol{[2, 4]}$  & $[2, 2]$  & 2 & 0.32 & 214.3 & S \\
& blobs (varied $\sigma$)            &  2 & 3 & $[2, 2]$  & $[2, 2]$  & 2 & 0.56 & 5.8   & S \\
& two spirals                        &  2 & 2 & $\boldsymbol{[1, 20]}$ & $\boldsymbol{[1, 18]}$ & 20 & 0.19 & --   & N \\
\midrule
\multirow{8}{*}{high-$D$}
& blobs ($50$D, easy)                & 50 & 6 & $\boldsymbol{[6, 6]}$  & $\boldsymbol{[6, 221]}$ & 6 & 0.99 & 1.0  & T \\
& blobs ($50$D, hard)                & 50 & 6 & $[1, 5]$  & $\boldsymbol{[1, 170]}$ & 2 & 0.14 & 1.1  & T \\
& blobs ($50$D, anisotropic)         & 50 & 6 & $\boldsymbol{[6, 6]}$  & $\boldsymbol{[6, 123]}$ & 6 & 1.00 & 1.3  & T \\
& blobs ($100$D, easy)               &100 & 6 & $\boldsymbol{[6, 6]}$  & $\boldsymbol{[6, 244]}$ & 5 & 1.00 & 1.0  & T \\
& blobs ($100$D, hard)               &100 & 6 & $\boldsymbol{[6, 6]}$  & $\boldsymbol{[6, 189]}$ & 6 & 0.99 & 1.2  & T \\
& blobs ($200$D, easy)               &200 & 6 & $\boldsymbol{[6, 6]}$  & $\boldsymbol{[6, 284]}$ & 6 & 0.98 & 1.0  & T \\
& blobs ($200$D, hard)               &200 & 6 & $\boldsymbol{[6, 6]}$  & $\boldsymbol{[6, 181]}$ & 5 & 1.00 & 1.5  & T \\
& Swiss roll (in $20$D)              &  3 & 4 & $\boldsymbol{[1, 4]}$  & $[1, 2]$   & 2 & 0.00 & --   & N \\
\midrule
\multirow{2}{*}{hierarchical}
& $3\times 3$ blobs ($2$D)           &  2 & 9 & $\boldsymbol{[3, 38]}$ & $\boldsymbol{[3, 47]}$  & 3 & 0.21 & 2.4  & T \\
& $3\times 3$ blobs ($10$D)          & 10 & 9 & $[5, 6]$  & $[5, 6]$   & 5 & 0.58 & 5.4  & S \\
\midrule
\multirow{7}{*}{class imbalance}
& blobs ($90/10$)                    &  2 & 2 & $\boldsymbol{[2, 2]}$  & $\boldsymbol{[2, 2]}$   & 2 & 0.97 & 21.8 & S \\
& blobs ($80/15/5$)                  &  2 & 3 & $\boldsymbol{[3, 3]}$  & $[2, 2]$   & 3 & 0.99 & 21.3 & S \\
& blobs ($60/30/10$)                 &  2 & 3 & $\boldsymbol{[3, 3]}$  & $\boldsymbol{[3, 3]}$   & 3 & 0.99 & 18.9 & S \\
& blobs (4-cluster, imbal.)          &  2 & 4 & $\boldsymbol{[4, 4]}$  & $[3, 3]$   & 4 & 0.99 & 10.1 & S \\
& blobs ($85/10/5$, $50$D)           & 50 & 3 & $[1, 1]$  & $\boldsymbol{[1, 270]}$ & 1 & 0.00 & 1.0  & T \\
& blobs (3-cluster + bg $10\%$)      &  2 & 3 & $[1, 2]$  & $[1, 2]$   & 2 & 0.79 & 6.7  & S \\
& moons ($80/20$)                    &  2 & 2 & $\boldsymbol{[2, 4]}$  & $\boldsymbol{[2, 3]}$   & 2 & 0.99 & 2.8  & T \\
\midrule
\multirow{4}{*}{adversarial}
& touching Gaussians ($2$D)          &  2 & 3 & $[1, 2]$  & $[1, 1]$   & 2 & 0.00 & --   & N \\
& touching Gaussians ($10$D)         & 10 & 3 & $[1, 1]$  & $\boldsymbol{[1, 74]}$  & 1 & 0.00 & 1.0  & T \\
& uneven blobs ($2$D)                &  2 & 3 & $[2, 2]$  & $[2, 2]$   & 2 & 0.87 & 7.4  & S \\
& helix + plane + sphere             &  3 & 3 & $[2, 2]$  & $[2, 2]$   & 2 & 0.57 & 31.6 & S \\
\bottomrule
\end{tabular}
\end{table}

\paragraph{Reading the catalog.}
A few patterns cut across families. The classic shapes are tight when the geometry is clean. The noise and contamination sweeps show the bracket widening or dropping its lower endpoint as the separation becomes less defensible. The high-dimensional rows show why the primary and mass brackets are different diagnostics: the primary can stay tight at the intended count while the mass companion exposes many small fragments in Euclidean high dimension. The hierarchical and adversarial rows are boundary cases. Their wide or collapsed brackets are not failures to tune a parameter; they report that several resolutions, or no separated resolution, are defensible at the sampled scale.

\begin{table}[h]
\centering
\small
\caption{Per-family aggregates on the six synthetic families with a single-scale target interpretation. Informativeness is bracket coverage divided by (median width $+\,1$). The best-baseline ARI is the maximum across the DBSCAN and HDBSCAN parameter grids, OPTICS, BIRCH, and the $K^\star$-given runs of KMeans, GMM, Ward, and Spectral. Hierarchical and adversarial families are left out of this aggregate because they are designed to violate the single-scale assumptions behind the score.}
\label{tab:per-family}
\begin{tabular}{lc cccc cc}
\toprule
& & \multicolumn{4}{c}{Informativeness} & \multicolumn{2}{c}{ARI (mean)} \\
\cmidrule(lr){3-6} \cmidrule(lr){7-8}
family & $n_{\mathrm{ds}}$ & MBC & HDB-grid & DBS-grid & best & MBC & best baseline \\
\midrule
classic shapes      & 3 & \textbf{1.00} & 0.06 & 0.14 & MBC & 0.99 & 1.00 \\
additive noise      & 5 & 0.27 & 0.08 & \textbf{0.33} & DBS & 0.52 & 0.85 \\
bg.\ contamination  & 6 & \textbf{0.44} & 0.12 & 0.25 & MBC & 0.59 & 0.93 \\
varied scale        & 3 & \textbf{0.22} & 0.20 & 0.20 & MBC & 0.36 & 0.70 \\
high-$D$ blobs      & 8 & 0.88 & \textbf{1.00} & 0.11 & HDB & 0.76 & 0.96 \\
class imbalance     & 7 & \textbf{0.71} & 0.21 & 0.10 & MBC & 0.82 & 0.98 \\
\bottomrule
\end{tabular}
\end{table}

\section{Mass-bounded bracket sensitivity}
\label{app:mass-bracket}

This section compares the primary bracket with $K_{\mathrm{mass},\gamma}$ at $\gamma\in\{0.90,0.95,0.98\}$ on the same sweep. Table~\ref{tab:mass-bracket-primary} gives suite-level coverage and width, and Table~\ref{tab:mass-bracket-divergence} gives the datasets driving the differences.

\begin{table}[h]
\centering
\small
\caption{Per-suite summary statistics for each bracket variant. Coverage is computed only on rows with a recorded $K^\star$. Median width is the median of $K_{\mathrm{high}} - K_{\mathrm{low}}$, and informativeness is coverage divided by median width plus one.}
\label{tab:mass-bracket-primary}
\begin{tabular}{llccc}
\toprule
suite & bracket & coverage & median width & informativeness \\
\midrule
\multirow{4}{*}{synth (38)}
  & $K_{\mathrm{big}}$ (primary)          & 0.68 & 0   & 0.68 \\
  & $K_{\mathrm{mass},0.90}$               & 0.58 & 0.5 & 0.39 \\
  & $K_{\mathrm{mass},0.95}$               & 0.66 & 1   & 0.33 \\
  & $K_{\mathrm{mass},0.98}$               & 0.76 & 2   & 0.25 \\
\midrule
\multirow{4}{*}{real labeled (9)}
  & $K_{\mathrm{big}}$ (primary)          & 0.22 & 1  & 0.11 \\
  & $K_{\mathrm{mass},0.90}$               & 0.44 & 3  & 0.11 \\
  & $K_{\mathrm{mass},0.95}$               & 0.56 & 13 & 0.04 \\
  & $K_{\mathrm{mass},0.98}$               & 0.89 & 37 & 0.02 \\
\midrule
\multirow{4}{*}{neuro (2)}
  & $K_{\mathrm{big}}$ (primary)          & 1.00 & 3.5 & 0.22 \\
  & $K_{\mathrm{mass},0.90}$               & 1.00 & 4   & 0.20 \\
  & $K_{\mathrm{mass},0.95}$               & 1.00 & 5.5 & 0.15 \\
  & $K_{\mathrm{mass},0.98}$               & 1.00 & 8   & 0.11 \\
\bottomrule
\end{tabular}
\end{table}

On synthetic data, the primary has the highest informativeness. On labeled real data, increasing $\gamma$ raises coverage but rapidly widens the interval; all variants cover the two neuroscience targets, with width increasing in $\gamma$.

\begin{table}[h]
\centering
\small
\caption{Representative divergences between the primary bracket and $K_{\mathrm{mass},0.95}$.}
\label{tab:mass-bracket-divergence}
\begin{tabular}{llcc}
\toprule
pattern & dataset & primary & mass companion \\
\midrule
long-tail real & Wine & $[1,2]$ & $[1,12]$ \\
sub-$\gamma$ minority & blobs ($80/15/5$) & $[3,3]$ & $[2,2]$ \\
high-$D$ fragmentation & $50$D blobs (easy) & $[6,6]$ & $[6,221]$ \\
\bottomrule
\end{tabular}
\end{table}

\paragraph{Choice of $\gamma$.}
The value $\gamma=0.90$ often drops true minorities carrying between $5\%$ and $10\%$ of the sample, while $\gamma=0.98$ admits long fragmentation tails and greatly widens the real-data intervals. We therefore use $\gamma=0.95$ as the intermediate coverage diagnostic. The run-length variant can suppress persistent high-dimensional fragmentation but performs poorly on short real and neuroscience sweeps, so it is recorded but not recommended by default.

\section{Neuroscience filtration: per-dataset results}
\label{app:neuro-behavior}

We report the full filtration behavior on the retinal ganglion cell and V1 datasets in Fig.~\ref{fig:retina_filtration}. The top row shows the first two diffusion-map components, colored by ground-truth labels where known. The bottom row shows the filtration $K(k)$ across the uncertainty zone, with the bracket and $k^\star$ marked.
\begin{figure}[h!]
\centering
\includegraphics[width=0.85\linewidth]{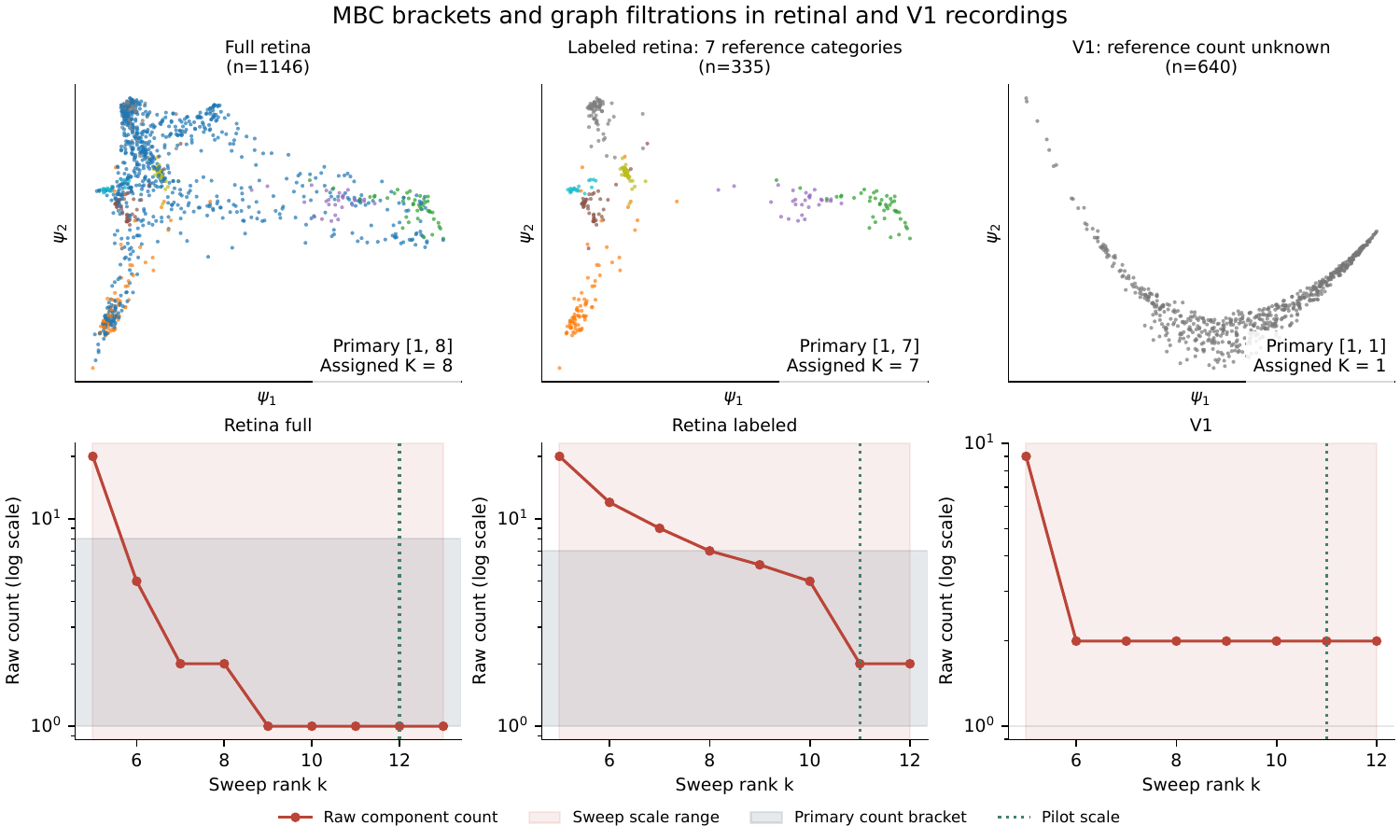}
\caption{Retinal ganglion cells and V1 in diffusion-map space. \emph{Top:} first two diffusion-map components with ground-truth labels where known. \emph{Bottom:} filtration $K(k)$ across the uncertainty zone; the shaded band marks the bracket and the dotted line marks $k^\star$.}
\label{fig:retina_filtration}
\end{figure}

\end{document}